\documentclass[lettersize,journal]{IEEEtran}
\usepackage{amsmath,amsfonts}
\usepackage{amssymb}
\usepackage{algorithmic}
\usepackage{algorithm}
\usepackage{array}
\usepackage{subfig}
\usepackage{textcomp}
\usepackage{stfloats}
\usepackage{url}
\usepackage{bm}
\usepackage{verbatim}
\usepackage{graphicx}
\usepackage{cite}
\usepackage{footnote}
\allowdisplaybreaks[4]

\usepackage{geometry}
\usepackage{enumerate}
\usepackage{stfloats}
\usepackage{amsthm}
\usepackage{lettrine}
\newtheorem{assumption}{Assumption}
\newtheorem{theorem}{Theorem}
\newtheorem{lemma}{Lemma}

\newtheorem{remark}{Remark}
\DeclareMathOperator*{\Minimize}{minimize}
\DeclareMathOperator{\st}{subject\ to}

\usepackage{color}
\definecolor{myc1}{rgb}{0,0,1}
\newcommand{\rev}{\textcolor{black}}

\begin{document}
\newgeometry{top=0.75in, bottom=1.0in, left=0.625in, right=0.625in}

\title{Prioritizing Gradient Sign Over Modulus: An Importance-Aware Framework for Wireless Federated Learning}

\author{Yiyang Yue, Jiacheng Yao, Wei~Xu,~\IEEEmembership{Fellow, IEEE,} 
Zhaohui Yang,~\IEEEmembership{Member, IEEE,} George~K.~Karagiannidis,~\IEEEmembership{Fellow,~IEEE,} and Dusit Niyato,~\IEEEmembership{Fellow, IEEE}

\thanks{Part of this work is presented in \textit{IEEE ICC 2025}\cite{11162025}. \textit{(Corresponding author: Wei Xu)}}
\thanks{Y. Yue, J. Yao, and W. Xu are with the National Mobile Communications Research Laboratory (NCRL), Southeast University, Nanjing 210096, China. J. Yao and W. Xu are also with Purple Mountain Laboratories, Nanjing 211111, China (e-mail: \{yyyue\_, jcyao, wxu\}@seu.edu.cn).}
\thanks{Z. Yang is with the Zhejiang Lab, Hangzhou 311121, China, and also with the College of Information Science and Electronic Engineering, Zhejiang University, Hangzhou, Zhejiang 310027, China (e-mail: yang\_zhaohui@zju.edu.cn).}
\thanks{G. K. Karagiannidis is with the Department of Electrical and
 Computer Engineering, Aristotle University of Thessaloniki, Greece (email:
 geokarag@auth.gr).}
\thanks{D. Niyato is with the School of Computer Science and Engineering, Nanyang Technological University, Singapore 639798 (e-mail: dniyato@ntu.edu.sg).}

\vspace{-1.0cm}

}



\maketitle

\begin{abstract}
Wireless federated learning (FL) facilitates collaborative training of artificial intelligence (AI) models to support ubiquitous intelligent applications at the wireless edge. However, the inherent constraints of limited wireless resources inevitably lead to unreliable communication, which poses a significant challenge to wireless FL. To overcome this challenge, we propose Sign-Prioritized FL (SP-FL), a novel framework that improves wireless FL by prioritizing the transmission of important gradient information through uneven resource allocation. Specifically, recognizing the importance of descent direction in model updating, we transmit gradient signs in individual packets and allow their reuse for gradient descent if the remaining gradient modulus cannot be correctly recovered. To further improve the reliability of transmission of important information, we formulate a hierarchical resource allocation problem based on the importance disparity at both the packet and device levels, optimizing bandwidth allocation across multiple devices and power allocation between sign and modulus packets. To make the problem tractable, the one-step convergence behavior of SP-FL, which characterizes data importance at both levels in an explicit form, is analyzed. We then propose an alternating optimization algorithm to solve this problem using the Newton-Raphson method and successive convex approximation (SCA). Simulation results confirm the superiority of SP-FL, especially in resource-constrained scenarios, demonstrating up to 9.96\% higher testing accuracy on the CIFAR-10 dataset compared to existing methods.
\end{abstract}

\vspace{-0.6cm}

\section{Introduction}
In the sixth-generation (6G) mobile networks, the integration of AI with communications is transforming wireless networks from basic device connectivity to intelligent connectivity \cite{10604756,10183789,11016266}. With devices at the network edge generating a huge amount of data, machine learning (ML) algorithms have been developed to exploit massive data sets and optimize system performance \cite{xu2023toward}. However, most current ML algorithms focus on centralized paradigms that rely heavily on significant raw data transfers and substantial computational resources at the central server \cite{10024766,10678839}. These paradigms also pose privacy risks and impose significant burdens on network resources \cite{10064038,10857353}. \rev{Spurred by the need to shift intelligence from centralized clouds to distributed edges, an influential line of seminal studies has emerged~\cite{11180854,11037631}. The work~\cite{11180854} innovatively proposes a comprehensive framework for 6G-oriented edge intelligence deployment and resource management, while \cite{11037631} further establishes a systematic methodology for enabling efficient edge intelligence in resource-constrained IoT environments.} Building on these advances, federated learning (FL) has emerged as a promising distributed solution that realizes collaborative training among devices without sharing the raw data \cite{10605604,10767214,jiang2025towards}. Specifically, during the FL training process, a common global model is trained on distributed devices before being transmitted to the parameter service (PS) for aggregation. By transmitting model parameters instead of raw data, FL successfully reduces communication overhead while enhancing user privacy.

\rev{Prior studies have offered important perspectives on deploying FL over wireless networks, where model exchanges occur through flexible wireless links. AirComp (over-the-air computation) enables simultaneous transmissions from multiple devices and allows the server to directly obtain aggregated model updates, making it an attractive solution for wireless FL. Building on this advantage, the works~\cite{10032291,9843892} propose RIS-assisted AirComp-based FL frameworks and corresponding transceiver designs, which effectively mitigate distortion caused by channel fading and noise.} However, considering that analog AirComp is not fully compatible with existing digital communication systems, many studies have shifted toward digital communication-based FL, which aligns better with current network infrastructures. Notably, most wireless FL methods based on digital communication still rely on the assumption of reliable transmission between the PS and the devices. However, reliable transmission can hardly be guaranteed in wireless networks due to practical constraints of transmit power budget and available bandwidth \cite{10261509,9264742}. In addition, a rapid increase in the number of participating entities and ML model sizes further exacerbates the constraints imposed by limited wireless resources.

To address the issue of unreliable transmission, existing works \cite{10142015,10185584,9716792} have investigated compensation methods for unsuccessfully received models that mitigate the impact of unreliable transmission on FL performance from an algorithmic perspective. In \cite{10142015}, a global model reuse scheme was developed that mitigates the negative impact of packet loss by replacing erroneous local models with the most recent global model. Similarly, a similarity-based compensation scheme was proposed in \cite{10185584}, which exploits the similarity between local models to correct the biased global model estimation caused by link failures. In addition, the authors in \cite{9716792} replaced lost packets with local parameters to enable decentralized FL based on the User Datagram Protocol (UDP), which approaches the optimal convergence rate with error-free transmission. Despite the effectiveness of these methods, they did not consider improving transmission reliability but rather passively compensated for errors after transmission failures occurred. Therefore, it is difficult to fundamentally address the challenge of unreliable communication as resources become scarcer.


Alternatively, resource allocation has become another key technique to address the challenges of unreliable transmission in resource-constrained wireless FL systems \cite{9210812,9611373,10368103}. For example, in \cite{9210812}, power allocation and resource block allocation were jointly optimized, effectively mitigating the adverse effects of packet errors due to limited resources. In \cite{9611373}, it was pointed out that maintaining uniform transmission failure probabilities across clients could mitigate the negative effects of unreliable communication, and accordingly, wireless resource allocation and transmission rates were jointly optimized to achieve this uniformity. Furthermore, considering the fluctuating long-term channel state, the authors of~\cite{10368103} analyzed the influence of transmission errors on FL performance at each iteration, enabling a dynamic resource allocation strategy that balances latency and energy constraints. 

These methods have proven effective in mitigating the adverse effects of unreliable transmission. However, in wireless FL, the transmitted data has different relevance to the learning tasks, leading to \textit{heterogeneous data importance}. Compared to uniformly improving transmission reliability, prioritizing the protection of critical data can more effectively improve FL performance, especially under critical wireless resource constraints. 
Several studies have developed resource allocation strategies that take into account the importance of transmitted data, with a particular focus on device scheduling, and have proposed various methods for evaluating data importance\cite{10038617,10145043,10659225}.
In~\cite{10038617}, a dynamic device scheduling mechanism was proposed where the norm of the transmitted gradient was used as a metric to evaluate the quality of scheduled devices. Furthermore, in~\cite{10145043}, a greedy algorithm-based scheduling strategy was introduced, where important gradients are defined as those that, after aggregation, most closely approximate the global gradient when all devices participate. In addition to considering the characteristics of the transmitted data itself, the authors of~\cite{10659225} proposed a probabilistic scheduling scheme that incorporates the divergence between the global model and the updated local model into the scheduling decision to evaluate the importance of the data to be transmitted.

While the above methods evaluated the importance of data from different perspectives, they rely primarily on device-level characteristics, overlooking the fact that transmitted data from the same device may have heterogeneous data importance. Furthermore, limiting device participation compromises the generalizability of the global model, which ultimately degrades the FL performance. Therefore, it is crucial to prioritize more critical data at a finer granularity, such as within each gradient, while still ensuring broad device participation. However, to our knowledge, few studies have investigated unequal data protection within transmitted data. A notable exception is~\cite{9272666}, which introduces a one-bit quantization scheme that preserves only the signs of gradients, highlighting the importance of gradient signs over their moduli. This observation underscores the existence of finer-grained differences in meaning and motivates our proposed approach.


Inspired by the above observations, in this paper, we propose a sign-prioritized FL (SP-FL) method that accounts for the heterogeneity of data importance, with a special emphasis on prioritizing the transmission of gradient signs over their moduli. The main contributions of this paper are summarized below.

\begin{itemize}
    \item \rev{To enhance the performance of wireless FL in resource-constrained networks, we propose a novel SP-FL framework. In this framework, the importance of different gradient components are identified, and resources are allocated with priority given to more important components, which is referred to as \textit{“importance-aware”}.} Specifically, recognizing the importance of gradient direction in model updating, we packetize gradient signs separately from moduli, allowing the reuse of correctly received sign packets with erroneous modulus packets through a compensatory modulus vector at the PS. Based on the sign-modulus separation strategy, we propose a hierarchical resource allocation scheme that formulates a long-term optimization problem over iterations to improve FL training performance by optimizing system bandwidth allocation at the device level and power allocation at the packet level.


\item To solve the long-term optimization problem, we account for channel fluctuations and data importance variations across iterations, decomposing the problem into a series of per-iteration optimization problems. To address implicit objectives, we analyze the one-step convergence behavior of SP-FL, which quantitatively captures the impact of transmitted data at both the device and packet levels per iteration, and show that the successful transmission of sign packets is more critical to convergence. Building on the analytical results, the subproblems are equivalently expressed in an explicit form. Then, we use an alternating optimization approach to solve these subproblems, where power allocation is efficiently handled by the Newton-Raphson method, and bandwidth allocation is addressed by successive convex approximation (SCA).
    
    \item We perform extensive experiments to validate the theoretical analysis and the effectiveness of the proposed SP-FL. It is shown that the analysis of the convergence behavior is in close agreement with the experimental results. It is also shown that the SP-FL outperforms the baselines in terms of both testing accuracy and convergence rate. The SP-FL achieves nearly error-free training performance with abundant wireless resources and exhibits high robustness when resources are critically limited.
\end{itemize}

The rest of this paper is organized as follows: Section \ref{system_model} describes the system model and the proposed SP-FL while formulating the resource allocation problem. In Section \ref{problem_formulation_and_convergence_analysis}, we provide the one-step convergence analysis of SP-FL to transform the problem into a tractable form.  To solve the problem, an alternating algorithm for joint optimization of bandwidth and transmit power allocation is proposed in Section \ref{hierarchical_resource_management_for_resource_constrained_FL}. Simulation results and conclusions are given in Sections \ref{Numerical_Results}
and \ref{conclusion}, respectively.

\textit{Notations:} We use boldface lowercase letters to represent vectors. The set of all real numbers is denoted by $\mathbb{R}$. The superscript $(\cdot)^T$ denotes the transpose operation. The operator ``$\triangleq$'' represents ``defined as equal to,'' while $|\cdot|$ denotes the size of a set, $\nabla(\cdot)$ represents the gradient operator, $\langle \cdot, \cdot \rangle$ denotes the inner product operator, and $\Vert\cdot\Vert$ is used for the $\ell_2$ norm. \rev{We use $\lg(\cdot)$ to denote the base-10 logarithm.} A circularly symmetric complex Gaussian distribution is represented by $\mathcal{CN}$, and $\mathbb{E}[\cdot]$ denotes the expectation operation. The abbreviation ``w.p.'' represents ``with probability.''

\section{System Model and The Proposed Strategy}
\label{system_model}
A typical wireless FL system is considered, which comprises a PS coordinating the learning process across $K$ devices. The global model, represented by the parameter vector $\mathbf{w} \in \mathbb{R}^l$, is trained collaboratively across these devices, where $l$ is the dimension of the global model $\mathbf{w}$. 

\vspace{-0.3cm}
\subsection{Federated Learning Model}
For the device $k$, $\forall k\in \mathcal{K}=\{1,\ldots,K\}$, the local dataset is denoted by ${\mathcal{D}}_{k}$ and consists of multiple data samples. The $i$-th data sample in ${\mathcal{D}}_{k}$ is represented as $(\mathbf{x}_{i},y_{i})$, where $\mathbf{x}_{i}$ denotes the input feature vector, $y_{i}$ represents the corresponding label, $i=1,\ldots,D_k$, and $D_k$ is the size of dataset $\mathcal{D}_{k}$. The local loss function of the model $\mathbf{w}$ on ${\mathcal{D}}_{k}$ is defined as
\begin{equation}
F_k(\mathbf{w})
= \frac1{{D}_k}\sum_{(\mathbf{x}_i,y_i)\in{\mathcal{D}}_k}f(\mathbf{x}_i,y_i;\mathbf{w}),
\label{local_loss}
\end{equation}
where $f(\mathbf{x}_i,y_i;\mathbf{w})$ represents the sample-wise loss function with respect to $(\mathbf{x}_{i},y_{i})$. 

Without loss of generality, as in \cite{10142015,9272666}, we assume that the sizes of all local datasets are the same, i.e., $D_1 = D_2 = \cdots = D_K$. Then, the global loss function associated with all distributed local datasets is defined as 
\begin{equation}
F(\mathbf{w}) \triangleq \frac {\sum_{k=1}^K {D}_{k} F_k(\mathbf{w})}{\sum_{k=1}^K {D}_{k}} = \frac1{K} \sum_{k=1}^K F_k(\mathbf{w}).
\label{global_loss_function}
\end{equation}
The objective of the FL algorithm is to find the optimal global model parameter, denoted by $\mathbf{w}^*$, which minimizes the global loss function, i.e.,
\vspace{-0.3cm}
\begin{equation}
\mathbf{w}^*=\arg\min_\mathbf{w}F(\mathbf{w}).
\label{w_opt}
\end{equation}

To handle the optimization problem in (\ref{w_opt}), the FL algorithm performs model training iteratively, with the gradient descent (GD) algorithm at each device to compute local gradients. In particular, in the $n$-th iteration, $\forall n\in \mathcal{N}=\{1,\ldots,N\}$, the FL algorithm involves the following steps.

\begin{enumerate}[1)]
\item \textit{Model Broadcasting}: The PS broadcasts the latest global model parameter $\mathbf w_n$ to all devices.
\item \textit{Local Gradient Computing}: After receiving $\mathbf w_n$, devices compute the local gradients in parallel. The local gradient $\mathbf{g}_{k,n} \in \mathbb{R}^l$ at device $k$ in this iteration is given by
\vspace{-0.2cm}
\begin{equation}
\mathbf{g}_{k,n} \triangleq \nabla F_k(\mathbf{w}_n)
=\frac1{{D}_k}\sum_{(\mathbf{x}_i,y_i)\in{\mathcal{D}}_k}\nabla f(\mathbf{x}_i,y_i;\mathbf{w}_n).
\label{local_gradient}
\vspace{-0.2cm}
\end{equation}
\item \textit{Local Gradient Uploading}: Each device uploads its local gradient to the PS.
\item \textit{Global Model Updating}: The PS aggregates received local gradients according to
\begin{equation}
\mathbf{g}_n = \frac{\sum_{k=1}^K D_k \mathbf{g}_{k,n}}{\sum_{k=1}^K {D}_{k}}= \frac 1K\sum_{k=1}^K \mathbf{g}_{k,n},
\label{global_gradient}
\end{equation}
and the global model is updated by
\vspace{-0.2cm}
\begin{equation}
\mathbf{w}_{n+1} = \mathbf{w}_n - \eta \mathbf{g}_{n},
\label{global_update}
\vspace{-0.2cm}
\end{equation}
where $\eta$ is the learning rate. 
\end{enumerate}  
The above steps are repeated iteratively until convergence. 

\vspace{-0.3cm}
\subsection{Gradient Quantization Model}
During the local gradient uploading, quantization methods are employed for ease of transmission. Without loss of generality, we adopt the typical stochastic quantization method \cite{9611373} as an example.
\footnote{\rev{SP-FL differs from conventional compression or pruning methods fundamentally in design objectives. SP-FL focuses on enhancing transmission reliability of important data, while compression or pruning methods prioritize reducing data volume for transmission efficiency. These two paradigms are inherently complementary, and SP-FL can be integrated into FL systems which already employing model compression or pruning.}}
Assume that the modulus of the $i$-th parameter ${g}_{k,n}^{(i)}$ in the local gradient $\mathbf{g}_{k,n}$ is bounded, satisfying $| {g}_{k,n}^{(i)}|\in[\underline{g}_{k,n},\,\overline{g}_{k,n}]$. With $b$ bits for the quantization of $| {g}_{k,n}^{(i)}|$, we denote $\{c_{0},c_{1},\ldots,c_{2^{b}-1}\}$ as the knobs uniformly distributed in $[\underline{g}_{k,n},\,\overline{g}_{k,n}]$, which are represented by
\begin{equation}
c_u=\underline{g}_{k,n}+u\times\frac{\overline{g}_{k,n}-\underline{g}_{k,n}}{2^{b}-1},\,u=0,\ldots,2^{b}-1.
\end{equation}
Then, the parameter ${g}_{k,n}^{(i)}$ falling in $[c_u,c_{u+1})$ is quantized by
\begin{align}
\mathcal{Q}({g}_{k,n}^{(i)}) &= s({g}_{k,n}^{(i)})\cdot \mathcal{Q}_v({g}_{k,n}^{(i)}) \nonumber \\
 &= \begin{cases}
 s({g}_{k,n}^{(i)})\cdot c_u,\quad\text{w.p. }\frac{c_{u+1}-\mid {g}_{k,n}^{(i)}\mid}{c_{u+1}-c_u},\\
 s({g}_{k,n}^{(i)})\cdot c_{u+1},\text{w.p. }\frac{\mid {g}_{k,n}^{(i)}\mid-c_u}{c_{u+1}-c_u},&\end{cases}
\end{align}
where $s(\cdot)$ represents the signum function and $\mathcal{Q}_v(\cdot)$ returns the quantization of the modulus scalar. Hence, the local gradient $\mathbf{g}_{k,n}$ is quantized as $\mathcal{Q}(\mathbf{g}_{k,n}) = [\mathcal{Q}({g}_{k,n}^{(1)}),\ldots,\mathcal{Q}({g}_{k,n}^{(l)})]$ before transmission to the PS. Let $b_0$ be the total number of bits used to represent $\underline{g}_{k,n}$ and $\overline{g}_{k,n}$, and the total number of bits required for transmission is $\hat b = l(b + 1) + b_0$, with the additional one bit representing the sign bit.

\begin{figure*}[htbp] 
\centering 
\includegraphics[width=0.75\textwidth]{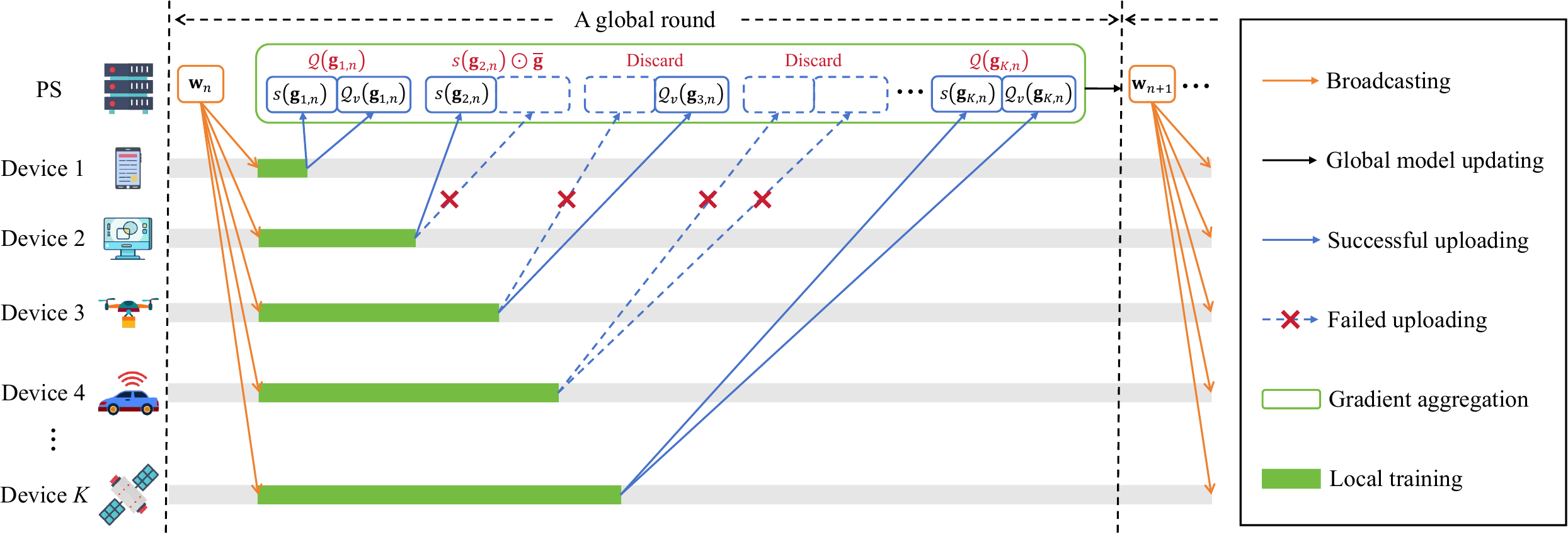} 
\caption{Illustration of the proposed SP-FL.}
\label{FL} 
\vspace{-0.5cm}
\end{figure*}

\subsection{Proposed Sign-prioritized FL Method}

For the downlink transmission in wireless FL systems, we assume that the PS has sufficient transmit power to ensure an error-free broadcast of global model parameters \cite{10552192,9611373}. However, the uplink transmission of wireless FL usually suffers from nonignorable errors due to limited wireless resources like power and bandwidth at devices. 

It is worth noting that, unlike traditional communication scenarios that pursue error-free information recovery, task-driven FL redefines the goal of communication to optimize training performance. To meet the objective and improve the performance of wireless FL under unreliable uplink communications, we propose a sign-prioritized FL strategy, which emphasizes the participation of critical information, especially gradient signs, in model updating through transmission design and resource allocation.


\subsubsection{Sign-Modulus Decoupled Transmission Strategy}

In the training process of FL, the sign of the gradient has a significant impact on model updating, as it determines the direction of gradient descent. Recognizing the importance of gradient sign, we propose a sign-modulus decoupled transmission strategy, as shown in Fig. \ref{FL}. In this strategy, each device transmits the gradient sign vector $s(\mathbf g_{k,n})$ as a single sign packet, while the modulus vector $\mathcal{Q}_v(\mathbf{g}_{k,n})$ and the bits used to represent $\underline{g}_{k,n}$ and $\overline{g}_{k,n}$ are sent in the modulus packet.

We assume that the total bandwidth available for uplink transmission of all devices is denoted by $B$. All devices upload gradients to the BS in a frequency-divisional manner, with each device's allocated bandwidth divided equally for the transmission of sign and modulus packets. In the $n$-th iteration, the transmit power budget at the device $k$ is denoted by $P_{k,n}$. The channel gain between the PS and device $k$ is modeled as $h_{k,n} d_k^{-\frac{\zeta}{2}}$, where $h_{k,n}$ captures the effects of small-scale fading, $d_k$ represents the distance between the PS and device $k$, and $\zeta$ is the path loss exponent. We assume that the channels experience independent Rayleigh fading, where $h_{k,n} \sim \mathcal{CN}(0,1)$. Under this model, the channel capacity for transmitting the sign packet for the device $k$ in the $n$-th iteration is evaluated as
\begin{equation}
    C_{k,n}^{(s)} = \frac{\beta_{k,n}B}{2}\log_2\left(1 + \frac{2\alpha_{k,n} P_{k,n}|h_{k,n}|^2d_k^{-\zeta}}{\beta_{k,n}B N_0} \right),
\vspace{-0.2cm}
\end{equation}
where $\alpha_{k,n}$ is the ratio of transmit power allocated to the sign packet transmission at the device $k$ in iteration $n$, and $\beta_{k,n}$ is the ratio of bandwidth allocation, while $N_0$ is the power spectrum density (PSD) of the additive noise.

Similarly, the channel capacity for the modulus packet transmission is evaluated as
\begin{equation}
    C_{k,n}^{(v)} \!=\! \frac{\beta_{k,n}B}{2}\!\log_2\left(1 \!+\! \frac{2(1-\alpha_{k,n}) P_{k,n}|h_{k,n}|^2d_k^{-\zeta}}{\beta_{k,n}B N_0} \right).
\end{equation}

To mitigate the risk of excessive delays caused by stragglers, we assume that all devices must transmit local gradients within the given latency $\tau$. Therefore, we assume the same transmission rate for all users, which serves as a truncation mechanism for users with poor channel conditions. Specifically, the general transmission rate of the sign packet is $R_s ={l}/{\tau}$, while the transmission rate of the modulus packet is $R_v = {(lb + b_0)}/{\tau}$.

Transmission errors occur when the transmission rate is greater than the channel capacity \cite{cover1991elements}. According to \cite{5703199}, the successful transmission probability of the signal packet in a Rayleigh fading and AWGN channel can be calculated as
\begin{align}
q_{k,n}(\alpha_{k,n},\beta_{k,n})= \begin{cases} \exp\left(\frac{H_s(\beta_{k,n})}{\alpha_{k,n}}\right),& \alpha_{k,n} \neq 0,\\0, & \alpha_{k,n} = 0,\end{cases}
\label{q}
\end{align}
where $H_s(\beta_{k,n})$ is defined by
\begin{equation}
H_s(\beta_{k,n})\triangleq\frac{\beta_{k,n}B N_0}{4 P_{k,n} d_k^{-\zeta}}\left(1-2^{\frac{2l}{\beta_{k,n}B\tau }}\right)\leq 0.
\label{H_s}
\end{equation}
Similarly, the successful transmission probability of the modulus packet is evaluated as
\begin{align}
p_{k,n}(\alpha_{k,n},\beta_{k,n})= \begin{cases} \exp\left(\frac{H_v(\beta_{k,n})}{1-\alpha_{k,n}}\right),& \alpha_{k,n} \neq 1,\\0, & \alpha_{k,n} = 1,\end{cases}
\label{p}
\end{align}
where 
\begin{equation}
H_v(\beta_{k,n})\triangleq\frac{\beta_{k,n}B N_0}{4 P_{k,n} d_k^{-\zeta}}\left(1-2^{\frac{2(lb+b_0)}{\beta_{k,n}B \tau}}\right)\leq 0.
\label{H_v}
\end{equation}

\subsubsection{Gradient Aggregation With Sign-Packet Reuse}
In traditional FL uplink transmission strategies, the signs and moduli of the same local gradient are transmitted as a single unit. Consequently, when transmission errors occur, both the sign and the modulus are indiscriminately discarded by the PS\cite{10551685,10253642}. However, according to \cite{9272666}, the global model update can be based solely on the signs of local gradients. To fully exploit the reliable part within even erroneously received data, we propose a sign-packet reuse strategy based on sign-modulus decoupled transmission, as shown in Fig. \ref{FL}. In particular, a cyclic redundancy check (CRC) mechanism is usually implemented at the PS to detect transmission errors. If the modulus packet is received erroneously while the corresponding sign packet is transmitted correctly, a compensatory modulus vector $\bar{\mathbf{g}}$ is employed to compensate for the erroneous modulus vector. Similar to gradient compensation algorithms, the vector $\bar{\mathbf{g}}$ can be defined as the modulus vector of the last global gradient \cite{9756506} or generated based on a shared random seed \cite{pmlr-v235-qin24a}, etc. Let $\hat{Q}_{v}(\mathbf g_{k,n})$ denote the estimate of ${Q}_{v}(\mathbf g_{k,n})$, which is characterized by
\begin{equation}
\hat{Q}_v\!(\mathbf{g}_{k,n}\!)\!=\!
\begin{cases}
{Q}_v\!(\mathbf{g}_{k,n}\!),
&\text{if }{Q}_v\!(\mathbf{g}_{k,n}\!) \text{ correctly received},\\
\overline{\mathbf{g}},
&\text{otherwise}.\end{cases}
\label{hat_Q}
\end{equation}
For the transmission of gradient signs, let $C(\mathbf g_{k,n})$ be an indicator for whether the sign packet from the device $k$ is correctly transmitted in iteration $n$, which is characterized by
\begin{equation}
C(\mathbf g_{k,n})=
\begin{cases}
1,& \text{if }s_(\mathbf g_{k,n}) \text{ correctly received},\\
0,& \text{otherwise}.
\end{cases}
\label{X}
\end{equation}
Specifically, $C(\mathbf g_{k,n}) = 0$ indicates that $s(\mathbf{g}_{k,n})$ is received by the PS erroneously. Since incorrect sign estimation leads to opposite gradient descent, packets from the corresponding device are rejected by the PS regardless of whether the modulus packets are transmitted correctly or not. Conversely, if $C(\mathbf g_{k,n}) = 1$, packets from device $k$ is accepted. \rev{It is worth noting that retransmitting the erroneous sign packet is also highly feasible, as each packet carries only one bit per gradient dimension and thus incurs minimal communication overhead. Given that sign packets are typically assigned higher transmission priority and more resources, their error rates become extremely low, making retransmissions rare. For simplicity, we omit the retransmission process in the subsequent design unless otherwise specified.}

To mitigate the impact of gradient discarding on the unbiasedness of gradient estimation, we multiply the coefficient $1/q_{k,n}$ such that $\mathbb{E}[C(\mathbf g_{k,n})] / q_{k,n} = 1$. Then, the aggregated global gradient is expressed as
\begin{equation}
\hat{\mathbf g}_n=\sum_{k=1}^K\frac{C(\mathbf g_{k,n}) \cdot s(\mathbf g_{k,n})\odot\hat{Q}_{v}(\mathbf g_{k,n})}{K q_{k,n}},
\label{hat_g}
\end{equation}
\rev{where $\odot$ is the Hadamard product operator.} With the global gradient estimated in (\ref{hat_g}), the global model updating in the $(n+1)$-th iteration is denoted by
\begin{equation}
\Tilde{\mathbf{w}}_{n+1} = \Tilde{\mathbf{w}}_{n}-\eta\mathbf{\hat{g}}_{n}.
\label{global_update_with_error}
\end{equation}

\subsubsection{Resource Allocation Problem Formulation for SP-FL}
\label{Problem_Formulation}
The proposed SP-FL enables signs and moduli to be decoupled, allowing the transmission of critical gradient information to be prioritized and thereby enhancing FL convergence. Moreover, considering the varying importance of gradients across devices, we formulate a hierarchical resource allocation problem to minimize the global loss. This problem entails prioritizing critical devices through bandwidth allocation and prioritizing essential packets, particularly sign packets, through power allocation, which is formulated as follows:
\begin{align}
\Minimize\limits_{\{\boldsymbol{\alpha}_n\!,\,\boldsymbol{\beta}_n\}_{n=0}^{N-1}}\quad & \mathbb E[F\left(\Tilde{\mathbf{w}}_{N}\right)] \label{Q_initial}\\
\st\quad  &0 \leq \alpha_{k,n}\leq 1, \,\forall k \in \mathcal{K},\forall n \in \mathcal{N} \tag{\ref{Q_initial}{a}} \label{Q_initiala}\\
& 0 \leq \beta_{k,n}< 1, \,\forall k \in \mathcal{K},\forall n \in \mathcal{N} \tag{\ref{Q_initial}{b}} \label{Q_initialb} \\
&\sum_{k=1}^{K} \beta_{k,n}\leq 1, \,\forall k \in \mathcal{K},\forall n \in \mathcal{N}, \tag{\ref{Q_initial}{c}} \label{Q_initialc}
\end{align}
where $\bm{\alpha}_n = [\alpha_{1,n},\cdots,\alpha_{K,n}]^T$ and $\bm{\beta}_n = [\beta_{1,n},\cdots,\beta_{K,n}]^T$. (\ref{Q_initiala}) are the transmit power allocation constraint, and (\ref{Q_initialb}) and (\ref{Q_initialc}) are system bandwidth allocation constraints.

Solving the problem in (\ref{Q_initial}) faces several challenges. First, it is almost impossible to derive an exact analytical expression for the global loss function with respect to $\{\boldsymbol{\alpha}_n,\boldsymbol{\beta}_n\}_{n=0}^{N-1}$. Secondly, since the problem spans a long-term horizon, it is difficult to represent $\mathbb E[F\left(\Tilde{\mathbf{w}}_{N}\right)]$ due to the unavailability of future channel state and gradient information. To overcome both difficulties, we analyze the one-step convergence of SP-FL in Section \ref{problem_formulation_and_convergence_analysis}.

\vspace{-0.3cm}
\section{One-Step Convergence Analysis of SP-FL}
\label{problem_formulation_and_convergence_analysis}
To transform the long-term optimization problem in (\ref{Q_initial}) into a tractable form, we first decompose it into a series of subproblems at each iteration. \rev{Considering that
\begin{align}
\large\mathbb{E}[F(\tilde{\mathbf{w}}_N)]-F(\tilde{\mathbf{w}}_0)=\sum_{n=0}^{N-1}\mathbb{E}[F(\tilde{\mathbf{w}}_{n+1})-F(\tilde{\mathbf{w}}_n)],
\label{Q_divide}
\end{align}
we minimize the expected one-step loss $\mathbb{E}[F(\tilde{\mathbf{w}}_{n+1})-F(\tilde{\mathbf{w}}_n)]$ given the current information available at round $n$. Although the per-iteration surrogate is only an approximation, it suffices for convergence due to the non-increasing expected loss sequence. Therefore,} in the $n$-th iteration, the corresponding subproblem of (\ref{Q_initial}) can be formulated as follows:
\begin{align}
\Minimize\limits_{\boldsymbol{\alpha}_n,\boldsymbol{\beta}_n}\quad & \mathbb E[F\left(\Tilde{\mathbf{w}}_{n+1}\right)]-F\left(\Tilde{\mathbf{w}}_{n}\right) \label{Q_iter}\\
\st\quad  &0 \leq \alpha_{k,n}\leq 1, \,\forall k \in \mathcal{K} \tag{\ref{Q_iter}{a}} \label{Q_itera}\\
& 0 \leq \beta_{k,n}< 1, \,\forall k \in \mathcal{K} \tag{\ref{Q_iter}{b}} \label{Q_iterb} \\
&\sum_{k=1}^{K} \beta_{k,n}\leq 1, \,\forall k \in \mathcal{K}. \tag{\ref{Q_iter}{c}} \label{Q_iterc}
\end{align}

To solve (\ref{Q_iter}), we first analyze the one-step convergence behavior of SP-FL in this section, which characterizes an upper bound of global loss reduction at each iteration.

\vspace{-0.3cm}
\subsection{Assumptions and Preliminary Lemma}
\label{Lemma_Assumptions}
To facilitate the analysis of the one-step convergence bound, the following common assumptions are necessary and presented here for completeness.

\begin{assumption}
\label{assumption2}
The local loss functions $F_k(\cdot)$ are $L$-smooth with the Lipschitz constant $L>0$, which follows
\begin{equation}
F_k(\mathbf{w}^{\prime})\leq F_k(\mathbf{w})+\nabla F_k(\mathbf{w})^T(\mathbf{w}^{\prime}-\mathbf{w}) + \frac L2\|\mathbf{w}^{\prime}-\mathbf{w}\|_2^2.
\label{L_smooth}
\end{equation}
\end{assumption}

\begin{assumption}
\label{assumption3}
The variance of the gap between the local gradient and the global gradient is bounded, which can be expressed as 
\begin{equation}
\| \mathbf{g}_{k,n}- \mathbf{g}_n\|^2 \leq \epsilon_{k,n}^2.    
\end{equation}
\end{assumption}

Assumptions \ref{assumption2} and \ref{assumption3} are widely used in literature for convergence analysis \cite{10552192,9272666}. Based on the assumptions, we present the following lemma regarding the strong convexity and Lipschitz smooth properties of the global loss function.

\begin{lemma}
Given that each local loss function $F_k(\cdot)$ is $L$-smooth, the global loss function $F(\cdot)$ inherits $L$-smoothness.
\end{lemma}

\begin{proof}
According to the definition of  $F(\cdot)$ in (\ref{global_loss_function}), $F(\cdot)$ is the linear combination of $F_k(\cdot)$. Under Assumptions \ref{assumption2}, any linear combination of $L$-smooth local loss functions also meets the condition (\ref{L_smooth}), completing the proof.
\end{proof}

Concerning the stochastic quantization scheme, we proceed to the formulation of the following lemma.

\begin{lemma}
\label{lemma}
With the stochastic quantization method, the value of the local gradient $\mathbf g_{k,n}$ is unbiasedly estimated under error-free transmission as
\begin{equation}
\mathbb{E}[Q(\mathbf g_{k,n})]=\mathbf g_{k,n},
\label{Q_exp}
\end{equation}
and the associated quantization error is bounded by
\begin{equation}
\mathbb{E}\left[\left\|{Q}(\mathbf g_{k,n})-\mathbf g_{k,n}\right\|^2\right] \leq 
\frac{l(\bar{g}_{k,n}-\underline{g}_{k,n})^{2}}{4(2^b-1)}\triangleq \delta^2_{k,n}.
\label{Q_var}
\end{equation}
\end{lemma}

\begin{proof}
The proof is found in \cite{9611373} and \cite{9277666}.
\end{proof}

\rev{The upper bound of the quantization error $\delta^2_{k,n}$ is determined by the maximum gradient modulus $\bar{g}_{k,n}$, the maximum gradient modulus $\underline{g}_{k,n}$, and the quantization bits $b$ at each device. Since these parameters can be directly obtained during local gradient computation, a tight bound on the quantization error can be efficiently computed. The resulting scalar $\delta^2_{k,n}$ is then fed back to the server alongside the quantized gradients, which incurs negligible communication overhead compared with the full gradient transmission.}

\subsection{Analysis of Convergence Behaviour}
\label{Convergence_Analysis}
Now, we are ready to derive the one-step convergence bound of the proposed SP-FL in the following theorem.

\begin{theorem}
\label{Theorem1}
Given the power allocation vector $\bm{\alpha}_n = [\alpha_{1,n},\cdots,\alpha_{K,n}]^T$ and the bandwidth allocation vector $\bm{\beta}_n = [\beta_{1,n},\cdots,\beta_{K,n}]^T$ in iteration $n$, the average decrement in the global loss function between two consecutive iterations $n$ and $n+1$ satisfies
\begin{align}
\mathbb E[F&\left(\Tilde{\mathbf{w}}_{n+1}\right)] - F\left(\mathbf{w}_{n}\right) \nonumber\\
 \leq&-\frac{\eta}{2}\|\mathbf g_{n}\|^{2}+\frac{\eta}{2}\|\bar{\mathbf g}\|^{2}+\frac{\eta}{K} \sum_{k=1}^{K} (\left\|\mathbf{g}_{k,n}\right\|^{2}+\epsilon_{k,n}^2-2\upsilon_{k,n})  \nonumber\\
& +\frac{\eta}{2K} \sum_{k=1}^{K} G(\alpha_{k,n},\beta_{k,n}),
\label{theorem1}
\end{align}
where $\upsilon_{k,n} = \langle \mathbf{g}_{k,n}, s(\mathbf{g}_{k,n})\odot\bar{\mathbf{g}}\rangle \geq 0$ and the expression of $G(\alpha_{k,n},\beta_{k,n})$ is presented in (\ref{G}) at the top of the next page.
\end{theorem}

\begin{figure*}
\begin{align}
G(\alpha_{k,n},\beta_{k,n}) 
\!\triangleq &\! \left(\!-4p_{k,n}\!+\!p_{k,n}^2 \!+\! L\eta\frac{p_{k,n}}{q_{k,n}}  \!\right)\!\left\|\mathbf{g}_{k,n}\!\right\|^{2}
\!+\! \left(\!-2p_{k,n} \!+\! p_{k,n}^2 \!+\! L\eta\frac{1-p_{k,n}}{q_{k,n}} \!\right)\!\left\|\bar{\mathbf{g}}\right\|^{2}
\!+\! \left(6p_{k,n}-2p_{k,n}^2 \right)\upsilon_{k,n} 
+ L\eta\frac{p_{k,n}}{q_{k,n}}\delta^2_{k,n} \nonumber\\
= & 2\!\left(\!-2\!\left\|\mathbf{g}_{k,n}\!\right\|^{2}\!-\!\left\|\bar{\mathbf{g}}\right\|^{2}\!+\!3\upsilon_{k,n}\!\right)\!\exp\left(\!\frac{H_v(\beta_{k,n})}{1-\alpha_{k,n}}\!\right)
\!+\!\left(\!\left\|\mathbf{g}_{k,n}\!\right\|^{2}\!+\!\left\|\bar{\mathbf{g}}\right\|^{2}\!-\!2\upsilon_{k,n}\!\right)\!\exp\left(\!\frac{2H_v(\beta_{k,n})}{1-\alpha_{k,n}}\!\right) \nonumber\\
&+\!L\eta(\left\|\mathbf g_{k,n}\right\|^{2}\!-\!\left\|\bar{\mathbf g}\right\|^{2}\!+\!\delta^2_{k,n})\!\exp\left(\!\frac{H_v(\beta_{k,n})}{1-\alpha_{k,n}}\!-\!\frac{H_s(\beta_{k,n})}{\alpha_{k,n}}\!\right)
\!+\!L\eta\left\|\bar{\mathbf g}\right\|^{2}\exp\left(\!-\frac{H_s(\beta_{k,n})}{\alpha_{k,n}}\!\right) \nonumber\\
\triangleq & A_{k,n}\!\exp\!\left(\!\frac{H_v\!(\!\beta_{k,n}\!)}{1\!-\!\alpha_{k,n}}\!\right)
\!+\! B_{k,n}\!\exp\!\left(\!\frac{2H_v\!(\!\beta_{k,n}\!)}{1\!-\!\alpha_{k,n}}\!\right)
\!+\! C_{k,n}\!\exp\!\left(\!\frac{H_v\!(\!\beta_{k,n}\!)}{1-\alpha_{k,n}}\!-\!\frac{H_s\!(\!\beta_{k,n}\!)}{\alpha_{k,n}}\!\right)
\!+\! D_{k,n} \exp\!\left(\!-\frac{H_s\!(\!\beta_{k,n}\!)}{\alpha_{k,n}}\!\right)\!
\label{G}
\end{align}
\hrulefill
\end{figure*}

\vspace{-0.2cm}
\begin{proof}
Please refer to Appendix \ref{proof_theorem_1}.
\end{proof}

\vspace{-0.1cm}
From \textbf{Theorem~\ref{Theorem1}}, we summarize the following key insights regarding the convergence bound and importance disparities.


\begin{remark}
\rm{We note that the norm of the gradient, $\left\|\mathbf{g}_{k,n}\right\|$, reflects the Euclidean distance between the updated local model at device $k$, $\mathbf{w}_{k,n+1}$, and the previous global model $\mathbf{w}_{n}$. When $\left\|\mathbf{g}_{k,n}\right\|$ is small, we argue that the contribution of $\mathbf{g}_{k,n}$ is insignificant given that $\mathbf{w}_{n}$ is known by the PS. As observed from (\ref{G}), increasing probabilities of success transmission, $p_{k,n}$ and $q_{k,n}$, for gradients with larger $\left\|\mathbf{g}_{k,n}\right\|$ leads to a more significant reduction in $G(\alpha_{k,n},\beta_{k,n})$ compared to gradients with smaller norms. In other words, more wireless resources should be provided to devices with gradients contributing more significantly to the global model updating.}
\end{remark}
\vspace{-0.2cm}


\begin{remark}
\rm{
\rev{In (\ref{theorem1}), the successful transmission probability of the sign packet, $q_{k,n}$, appears in the denominator of the convergence bound, while the modulus success probability $p_{k,n}$ only affects higher-order terms.} The convergence bound is decreased by $1/q_{k,n}$ as $q_{k,n}$ increases, being infinity if $q_{k,n}\to 0$. \rev{This indicates that sign errors directly leads to the divergence or oscillation of the global model, whereas modulus errors merely cause limited magnitude deviation.} Therefore, it is definitely preferable to prioritize the transmission of sign packets to ensure the convergence in resource-constrained scenarios.
}
\end{remark}

\begin{remark}
\rm{
\rev{The compensation vector $\bar{\mathbf{g}}$ is introduced primarily to alleviate the bias induced by packet loss, rather than to guarantee an unbiased gradient estimate. Its influence on the convergence behavior is captured through the terms $\|\bar{\mathbf{g}}\|^{2}$ and $\upsilon_{k,n}$, which affect the convergence bound but do not compromise the overall convergence guarantee of SP-FL. In particular, $\upsilon_{k,n}$ characterizes the dot-product similarity between the true modulus of $\mathbf{g}_{k,n}$ and the compensation vector $\bar{\mathbf{g}}$. As shown in (\ref{theorem1}), the resulting bias decreases with the increment of similarity between $\bar{\mathbf{g}}$ and the modulus of the actual gradient. In the ideal case, no bias is introduced when $\bar{\mathbf{g}}$ exactly matches the true modulus of $\mathbf{g}_{k,n}$.
}}
\end{remark}

\section{Hierarchical Resource Allocation for Resource-Constrained FL systems}
\label{hierarchical_resource_management_for_resource_constrained_FL}

Based on the one-step convergence analysis in Section~\ref{problem_formulation_and_convergence_analysis}, the optimization problem in (\ref{Q_iter}) can be formulated as the minimization of the right-hand side of (\ref{theorem1}), which is equivalent to minimizing $G(\alpha_{k,n},\beta_{k,n})$. Therefore, the problem in (\ref{Q_iter}) can be formulated as
\begin{align}
\Minimize\limits_{\boldsymbol{\alpha}_n,\boldsymbol{\beta}_n}\quad &\sum_{k=1}^{K} G(\alpha_{k,n},\beta_{k,n}) \label{Q1}\\
\st\quad&0 \leq \alpha_{k,n}\leq 1, \forall k \in \mathcal{K} \tag{\ref{Q1}{a}} \label{Q1a}\\
& 0 \leq \beta_{k,n} \leq 1, \forall k \in \mathcal{K} \tag{\ref{Q1}{b}} \label{Q1b} \\
&\sum_{k=1}^{K} \beta_{k,n}\leq 1, \forall k \in \mathcal{K}. \tag{\ref{Q1}{c}} \label{Q1c}
\end{align}

\rev{In each iteration, devices upload the gradient norm $\|\mathbf{g}_{k,n}\|$ required to solve problem (\ref{Q1}) after completing local training. Since the gradient norm $\|\mathbf{g}_{k,n}\|$ is a scalar, its reliable transmission can be easily ensured with negligible communication overhead~\cite{9382094}. For simplicity, we assume error-free transmission of $\|\mathbf{g}_{k,n}\|$.} However, it is still challenging to obtain the globally optimal solution of (\ref{Q1}) due to the coupled variables, i.e., $\boldsymbol{\alpha}_n$ and $\boldsymbol{\beta}_n$. To solve this subproblem, we adopt the alternating optimization method, in which $\boldsymbol{\alpha}_n$ and $\boldsymbol{\beta}_n$ are alternately optimized in an iterative manner.

\subsection{Device Transmit Power Allocation}
Given the bandwidth allocation vector $\boldsymbol{\beta}_n$, the problem in (\ref{Q1}) is simplified as 
\begin{align}
\Minimize\limits_{\boldsymbol{\alpha}_n}\quad &\sum_{k=1}^{K} G(\alpha_{k,n},\beta_{k,n}) \label{Qa}\\
\st\quad&(\ref{Q1}\rm b),\,(\ref{Q1}\rm c). \notag
\end{align}
For the optimization problem in (\ref{Qa}), we note that the variables $\{\alpha_{1,n},\ldots,\alpha_{K,n}\}$ within the vector $\boldsymbol{\alpha}_n$ are decoupled. Therefore, the optimization problem for $\boldsymbol{\alpha}_n$ can be decomposed into a series of subproblems for the optimization of scalar $\{\alpha_{k,n}\}_{k = 1}^K$ without loss of optimality. Then, the optimal ratio $\alpha_{k,n}$ can be obtained by the following lemma.
\begin{lemma}
\label{power_allocation_lemma}
Given the bandwidth allocation vector $\boldsymbol{\beta}_n$ of device $k$ in iteration $n$, the optimal transmit power allocation variable $\alpha_{k,n}^*$ is given by
\vspace{-0.2cm}
\begin{align}
\alpha_{k,n}^* = \mathop{\arg\min}\limits_{\alpha_{k,n}\in \{x_1,\dots,x_i,1\}} G(\alpha_{k,n},\beta_{k,n}),
\label{argmin}
\end{align}
where $0< x_1<\cdots<x_i<1$ and satisfies the equality in (\ref{opt_alpha}) at the top of the next page.
\end{lemma}
\begin{proof}
Please refer to Appendix \ref{proof_lemma}.
\end{proof}

\begin{figure*}
\begin{align}
&A_{k,n}\exp\left(\frac{H_v(\beta_{k,n})}{1-x_i}\right)\frac{H_v(\beta_{k,n})}{(1-x_i)^2} 
+B_{k,n}\exp\left(\frac{2H_v(\beta_{k,n})}{1-x_i}\right)\frac{2H_v(\beta_{k,n})}{(1-x_i)^2} \nonumber\\
&+C_{k,n}\exp\left(\frac{H_v(\beta_{k,n})}{1-x_i}-\frac{H_s(\beta_{k,n})}{x_i}\right)\left(\frac{H_v(\beta_{k,n})}{(1-x_i)^2} +\frac{H_s(\beta_{k,n})}{x_i^2} \right)
+D_{k,n}\exp\left(-\frac{H_s(\beta_{k,n})}{x_i}\right)\frac{H_s(\beta_{k,n})}{x_i^2}=0
\label{opt_alpha}
\end{align}
\hrulefill
\end{figure*}

Since the left-hand side of (\ref{opt_alpha}) is differentiable, the value of $x_i$ can be readily computed by, e.g., the Newton-Raphson method\cite{Newton1967TheMP}. The optimal solution $\alpha_{k,n}^*$ is then obtained by exhaustive search within the finite set $\{x_1,\ldots,x_i,1\}$.

\subsection{System Bandwidth Allocation}
Given the power allocation vector $\boldsymbol{\alpha}_n$, the problem in (\ref{Q1}) reduces to 
\vspace{-0.3cm}
\begin{align}
\Minimize\limits_{\boldsymbol{\beta}_n}\quad &\sum_{k=1}^{K} G(\alpha_{k,n},\beta_{k,n}) \label{Qb}\\
\st\quad&(\ref{Q1}\rm b),\,(\ref{Q1}\rm c). \notag
\end{align}
Note that the convexity of the subproblem in (\ref{Qb}) is determined by the signs of the coefficients $\{A_{k,n}, B_{k,n}, C_{k,n},D_{k,n}\}_{k=1}^K$ in the objective function. Both $B_{k,n}$ and $D_{k,n}$ are nonnegative, while the signs of $A_{k,n}$ and $C_{k,n}$ are to be determined. To proceed, we consider the following four cases separately.

$1)\  A_{k,n} \geq 0$, $C_{k,n} \geq 0:$ Let $\mathcal{K}_1$ represent the set of devices satisfying $A_{k,n} \geq 0$ and $C_{k,n} \geq 0$. Given that $H_v(\beta_{k,n})$ is concave, the initial three terms of $G(\alpha_{k,n},\beta_{k,n})$ in (\ref{G}) exhibit nonconvexity. To address this challenge, an auxiliary variable $t_{k,n}$ is introduced, constrained by the following condition:
\begin{equation}
t_{k,n} \geq \frac{H_v(\beta_{k,n})}{1-\alpha_{k,n}}.
\label{t_kn}
\end{equation}
Thus, for $k\in\mathcal{K}_1$, $G(\alpha_{k,n},\beta_{k,n})$ is equivalent to the following convex form
\begin{align}
G_1(k,n) \triangleq\, &A_{k,n}\exp(t_{k,n}) 
+ B_{k,n}\exp\left(2t_{k,n}\right) \nonumber\\
&+ C_{k,n}\exp\left(t_{k,n}-\frac{H_s(\beta_{k,n})}{\alpha_{k,n}}\right) \nonumber\\
&+ D_{k,n}\exp\left(-\frac{H_s(\beta_{k,n})}{\alpha_{k,n}}\right).
\label{G_1}
\end{align}

$2)\ A_{k,n} \geq 0$, $C_{k,n} < 0:$ Define $\mathcal{K}_2$ as the set of devices for which $A_{k,n} \geq 0$ and $C_{k,n} < 0$ hold. Similar to the above case~$1$), $t_{k,n}$ is employed to manage the nonconvexity of the first two terms of $G(\alpha_{k,n},\beta_{k,n})$ in (\ref{G}). Additionally, an auxiliary variable $z_{k,n}$ is defined to address the nonconvexity of the third term in $G(\alpha_{k,n},\beta_{k,n})$, subject to the following constraint
\begin{equation}
z_{k,n}\leq \exp\left(\frac{H_v(\beta_{k,n})}{1-\alpha_{k,n}}-\frac{H_s(\beta_{k,n})}{\alpha_{k,n}}\right),
\end{equation}
and for $k\in\mathcal{K}_2$, $G(\alpha_{k,n},\beta_{k,n})$ is equivalent to
\begin{align}
G_2(k,n) \triangleq \,& A_{k,n}\exp(t_{k,n}) 
+ B_{k,n}\exp\left(2t_{k,n}\right) \nonumber\\
&+ C_{k,n}z_{k,n} + D_{k,n}\exp\left(-\frac{H_s(\beta_{k,n})}{\alpha_{k,n}}\right).
\label{G_2}
\end{align}

$3)\ A_{k,n} < 0$, $C_{k,n} \geq 0:$ Let $\mathcal{K}_3$ denote the set of devices where $A_{k,n} < 0$ and $C_{k,n} \geq 0$. Here also, $t_{k,n}$ is introduced to handle the nonconvexity of the second and third terms of $G(\alpha_{k,n},\beta_{k,n})$ in (\ref{G}). Furthermore, an auxiliary variable $y_{k,n}$ is introduced to manage the nonconvexity of the first term in $G(\alpha_{k,n},\beta_{k,n})$, which is constrained by
\begin{equation}
y_{k,n} \leq \exp\left(\frac{H_v(\beta_{k,n})}{1-\alpha_{k,n}}\right),
\vspace{-0.2cm}
\end{equation}
and for $k\in\mathcal{K}_3$, $G(\alpha_{k,n},\beta_{k,n})$ is equivalent to
\begin{align}
G_3(k,n) \triangleq &A_{k,n}y_{k,n} 
+ B_{k,n}\exp\left(2t_{k,n}\right) \nonumber\\
&+ C_{k,n}\exp\left(t_{k,n}-\frac{H_s(\beta_{k,n})}{\alpha_{k,n}}\right) \nonumber\\
&+ D_{k,n}\exp\left(-\frac{H_s(\beta_{k,n})}{\alpha_{k,n}}\right).
\label{G_3}
\end{align}

$4)\ A_{k,n} < 0$, $C_{k,n} < 0:$ Finally, $\mathcal{K}_4$ is defined as the set of devices where both $A_{k,n} < 0$ and $C_{k,n} < 0$. Auxiliary variables $t_{k,n}$, $y_{k,n}$ and $z_{k,n}$ are introduced, and for $k\in\mathcal{K}_4$, $G(\alpha_{k,n},\beta_{k,n})$ is equivalent to
\begin{align}
G_4(k,n) \triangleq &A_{k,n}y_{k,n} 
+ B_{k,n}\exp\left(2t_{k,n}\right) \nonumber\\
&+ C_{k,n}z_{k,n}+ D_{k,n}\exp\left(-\frac{H_s(\beta_{k,n})}{\alpha_{k,n}}\right).
\label{G_4}
\end{align}

Thus far, the optimization problem in (\ref{Qb}) is rewritten as
\begin{align}
\Minimize\limits_{\boldsymbol{\beta}_n,\mathbf{t}_n,\mathbf{y}_n,\mathbf{z}_n}\quad &  \sum_{k\in\mathcal{K}_1} G_1(k,n)
+\sum_{k\in\mathcal{K}_2} G_2(k,n) \notag \\
&+\sum_{k\in\mathcal{K}_3} G_3(k,n) 
+\sum_{k\in\mathcal{K}_4} G_4(k,n) \label{Q_o}\\
\st\quad & (\ref{Q1}\text{b}),(\ref{Q1}\text{c}) \notag \\
& t_{k,n} \geq \frac{H_v(\beta_{k,n})}{1-\alpha_{k,n}}, \,\forall k \tag{\ref{Q_o}{a}} \label{Q_oa} \\
& z_{k,n}\!\leq\! \exp\!\left(\!\frac{H_v(\!\beta_{k,n}\!)}{1\!-\!\alpha_{k,n}}\!-\!\frac{H_s(\beta_{k,n})}{\alpha_{k,n}}\!\right),  \notag \\
&\quad\!\forall k\! \in \!\mathcal{K}_2\!\cup\!\mathcal{K}_4 \tag{\ref{Q_o}{b}} \label{Q_ob} \\
& y_{k,n} \!\leq\! \exp\!\left(\!\frac{H_v(\beta_{k,n})}{1-\alpha_{k,n}}\right), \notag \forall k\in \mathcal{K}_3\cup\mathcal{K}_4. \tag{\ref{Q_o}{c}} \label{Q_oc}  \end{align}

The problem in ($\ref{Q_o}$) is still hard to solve because of the nonconvexity of the constraints. To handle this difficulty, the SCA method is adopted. For (\ref{Q_oa}), it is a difference-of-convex (DC) constraint because $H_v(\beta_{k,n})$ is concave. In this way, (\ref{Q_oa}) is handled by iteratively upper bounding the right-hand side by its first-order Taylor expansion. For the $r$-th iteration of the SCA, we construct the upper bound
\begin{align}
\frac{H_v(\beta_{k,n})}{1-\alpha_{k,n}} \!\leq\!
\frac{H_v(\beta_{k,n}^{(r\!-\!1)})\!+\!H_v^{\prime}(\beta_{k,n}^{(r\!-\!1)})(\beta_{k,n}\!-\!\beta_{k,n}^{(r\!-\!1)})}{1-\alpha_{k,n}},
\end{align}
where superscript $(r-1)$ represents the value of the variable at the $(r-1)$-th iteration, and $H_v^{\prime}(\beta_{k,n})$ represents the first-order derivative with respect to $\beta_{k,n}$, which is denoted by
\begin{align}
H_v^{\prime}(\beta_{k,n})\triangleq
&\frac{B N_0}{4|h_0|^2 P_{k,n}  d_k^{-\zeta}}\left(1-2^{\frac{2(lb+b_0)}{\beta_{k,n}B\tau }}\right) \nonumber\\
&+\frac{(lb+b_0) N_0 \ln 2}{|h_0|^2 P_{k,n}  d_k^{-\zeta}\tau\beta_{k,n}} 2^{\left(\frac{2(lb+b_0)}{\beta_{k,n}B\tau}-1\right)}.
\end{align}
Above that, a convex subset of constraint (\ref{Q_oa}) is established as
\begin{align}
\frac{H_v(\!\beta_{k,n}^{(r-1)}\!)\!+\!H_v^{\prime}(\!\beta_{k,n}^{(r-1)}\!)(\!\beta_{k,n}\!-\!\beta_{k,n}^{(r-1)}\!)}{1-\alpha_{k,n}}\!-\!t_{k,n}\!\leq \!0, \forall k.
\label{H_v_SCA}
\end{align}

Next, for the constraint in (\ref{Q_ob}), we rewritten it as
\begin{align}
\ln z_{k,n} - \frac{H_v(\beta_{k,n})}{1-\alpha_{k,n}} + \frac{H_s(\beta_{k,n})}{\alpha_{k,n}} \leq 0,
\label{lnz1}
\end{align}
which is also a DC constraint. With the SCA method, a convex subset of (\ref{Q_ob}) is given by
\begin{align}
\ln &z_{k,n}^{(r-1)} + \frac{H_s(\beta_{k,n}^{(r-1)})+H_s^{\prime}(\beta_{k,n}^{(r-1)})(\beta_{k,n}-\beta_{k,n}^{(r-1)})}{\alpha_{k,n}} \nonumber\\
&+ \frac{z_{k,n} - z_{k,n}^{(r-1)}}{z_{k,n}^{(r-1)}} - \frac{H_v(\beta_{k,n})}{1-\alpha_{k,n}} \leq 0,\,\forall k \in \mathcal{K}_2\cup\mathcal{K}_4,
\label{lnz2}    
\end{align}
where $H_s^{\prime}(\beta_{k,n})$ is the first-order derivative of $H_s(\beta_{k,n})$ with respect to $\beta_{k,n}$. The expression of $H_s^{\prime}(\beta_{k,n})$ is denoted by
\begin{align}
H_s^{\prime}(\beta_{k,n})\triangleq
&\frac{B N_0}{4|h_0|^2 P_{k,n}  d_k^{-\zeta}}\left(1-2^{\frac{2l}{\beta_{k,n}B\tau }}\right) \nonumber\\
&+\frac{l N_0 \ln 2}{|h_0|^2 P_{k,n}  d_k^{-\zeta}\tau\beta_{k,n}} 2^{\left(\frac{2l}{\beta_{k,n}B\tau}-1\right)}.
\end{align}

Similar to (\ref{Q_ob}), constraint (\ref{Q_oc}) is then approximated by the following convex subset
\begin{align}
\ln{y_{k,n}^{(r\!-\!1)}} \!+\!\frac{y_{k,n}\!-\!y_{k,n}^{(r\!-\!1)}}{y_{k,n}^{(r\!-\!1)}}\!-\! \frac{H_v(\beta_{k,n})}{1-\alpha_{k,n}} \!\leq\! 0,\,\forall k\!\in\! \mathcal{K}_3\!\cup\!\mathcal{K}_4.
\label{lny2}
\end{align}

By exploiting the above convex approximations, a series of convex surrogate problems are formulated to locally approximate the problem in (\ref{Q_o}). Specifically, the convex surrogate problem in the $r$-th iteration is formulated as
\begin{align}
\Minimize\limits_{\boldsymbol{\beta}_n,\mathbf{t}_n,\mathbf{y}_n,\mathbf{z}_n}\quad &  \sum_{k\in\mathcal{K}_1} G_1(k,n)
+\sum_{k\in\mathcal{K}_2} G_2(k,n) \notag \\
&+\sum_{k\in\mathcal{K}_3} G_3(k,n) 
+\sum_{k\in\mathcal{K}_4} G_4(k,n) \label{Q5}\\
\st\quad& (\ref{Q1}\text{b}), (\ref{Q1}\text{c}), (\ref{H_v_SCA}), (\ref{lnz2}), (\ref{lny2}),\notag
\end{align}
which is readily solved by existing numerical convex program solvers, e.g., CVX tools\cite{cvx}. Upon obtaining the solution to (\ref{Q5}), variables $\boldsymbol{\beta}_n^{(r)}$, $\mathbf{t}_n^{(r)}$, $\mathbf{y}_n^{(r)}$, and $\mathbf{z}_n^{(r)}$ are updated and the iterative process advances to the $(r+1)$-th iteration. 

Once the iterative procedure converges, the obtained solution $\boldsymbol{\beta}_n$ is substituted into (\ref{argmin}) to compute $\boldsymbol{\alpha}_n$. This procedure is repeated until convergence is achieved. A summary of the proposed alternating algorithm for solving the problem (\ref{Q1}) is provided in Algorithm \ref{alg}, while the complete procedure of the proposed SP-FL algorithm is detailed in Algorithm \ref{PAFL}. \rev{From Algorithm \ref{PAFL}, we note that compared to traditional FL, SP-FL introduces additional operations in Steps 4 and 5. However, the added overhead is marginal, because the amount of transmitted data is relatively small. Specifically, in Step 4, each device only needs to upload one additional scalar, $\| \mathbf{g}_{k,n}\|$, along with its gradient update. In Step 5, the server broadcasts the vectors $\boldsymbol{\alpha}_{n}$ and $\boldsymbol{\beta}_{n}$. Their dimensions correspond to the number of participating devices, which is typically much smaller than the dimension of the neural network parameters.}

\begin{algorithm}[t]
	\caption{Proposed Alternating Algorithm to Solve (\ref{Q1})}
    \label{alg}
    \begin{algorithmic}[1] 
        \STATE Initialize $\boldsymbol{\beta}_{n}$ and convergence accuracy $\lambda$.
        \REPEAT
            \STATE Solve (\ref{opt_alpha}) with $\boldsymbol{\beta}_{n}$ and update $\boldsymbol{\alpha}_{n}$.
            \STATE Initialize $\{\boldsymbol{\beta}_{n}^{(0)},\mathbf{t}_n^{(0)},\mathbf{z}_n^{(0)},\mathbf{y}_n^{(0)}\}$ and $r=0$.
            \REPEAT
                \STATE Set $r=r+1$.
                \STATE Solve (\ref{Q5}) with $\{\boldsymbol{\alpha}_{n},\boldsymbol{\beta}_{n}^{(r-1)},\mathbf{t}_n^{(r-1)},\mathbf{z}_n^{(r-1)},\mathbf{y}_n^{(r-1)}\}$ and update $\{\boldsymbol{\beta}_{n}^{(r)},\mathbf{t}_n^{(r)},\mathbf{z}_n^{(r)},\mathbf{y}_n^{(r)}\}$.
            \UNTIL the objective value (\ref{Q5}) convergences.
        \UNTIL the objective value (\ref{Q1}) convergences.
        \STATE Output $\boldsymbol{\alpha}_{n}$ and $\boldsymbol{\beta}_{n}$.
    \end{algorithmic}
\end{algorithm}

\begin{algorithm}[t]
	\caption{Sign-Prioritized Federated Learning (SP-FL)}
    \label{PAFL}
    \begin{algorithmic}[1] 
        \STATE The PS initializes the global model $\mathbf{w}_0$.
        \FOR{iteration $n$}
        \STATE The PS broadcasts $\Tilde{\mathbf{w}}_n$ to all devices.
        \STATE Each device $k$ computes $\mathbf{g}_{k,n}$ and sends $\| \mathbf{g}_{k,n} \|$ to the PS.
        \STATE The PS computes $\boldsymbol{\alpha}_{n}$ and $\boldsymbol{\beta}_{n}$ by solving (\ref{Q1}) and broadcasts $\boldsymbol{\alpha}_{n}$ and $\boldsymbol{\beta}_{n}$ to devices.
        \STATE Each device $k$ transmits $\mathbf{g}_{k,n}$ to the PS through the sign and modulus packets.
        \STATE The PS aggregates received local gradients and updates the global FL model $\Tilde{\mathbf{w}}_{n+1}$ using (\ref{global_update_with_error}).
        \ENDFOR
    \end{algorithmic}
\end{algorithm}

\vspace{-0.3cm}
\subsection{Complexity Analysis}
From Algorithm \ref{alg}, the main complexity of solving the problem (\ref{Q1}) lies in solving the power allocation problem in (\ref{argmin}) and the bandwidth allocation problem in (\ref{Q5}). The computational complexity of the power allocation problem primarily arises from finding roots of (\ref{opt_alpha}) for each device. Determining the roots of (\ref{opt_alpha}) introduces a complexity of $\mathcal{O}(m\log_2(1/\gamma_1))$, where $m$ represents the number of subintervals within the solution range and $\gamma_1$ denotes the precision required for the Newton-Raphson method \cite{nocedal2006numerical}. Thus, the overall complexity of solving the power allocation problem is $\mathcal{O}(Km\log_2(1/\gamma_1))$. The bandwidth allocation problem in (\ref{Qb}) is solved by the SCA method. Since there are $K$ variables in problem (\ref{Qb}), the number of iterations required for the SCA method is $\mathcal{O}(\sqrt K \log_2(1/\gamma_2))$, where $\gamma_2$ is the required convergence accuracy of the SCA method. At each iteration, the complexity of solving the problem in (\ref{Q5}) is $\mathcal{O}(K^3)$\cite{LOBO1998193}. Therefore, the complexity of the SCA method for solving the problem in (\ref{Qb}) is $\mathcal{O}(K^{3.5} \log_2(1/\gamma_2))$. As a result, the total complexity of Algorithm \ref{alg} for solving problem (\ref{Q1}) is $\mathcal{O}(S Km\log_2(1/\gamma_1)+S K^{3.5} \log_2(1/\gamma_2))$, where $S$ is the number of iterations for Algorithm \ref{alg}.

\subsection{\rev{Low-Complexity Optimization Method}}
\rev{The aforementioned algorithm incurs substantial computational overhead when the number of devices $K$ becomes large, resulting in prolonged optimization and delayed convergence. To enhance the adaptability of SP-FL in large-scale device scenarios, we provide a low-complexity algorithm to reduce computational complexity.}

\rev{The computational complexity primarily stems from solving the bandwidth allocation subproblem (\ref{Qb}). Observing that the constraints in problem (\ref{Qb}) are relatively straightforward, we consider employing the interior-point penalty function method~\cite{roy2024elements} to transform problem (\ref{Qb}) into an unconstrained optimization problem:
\begin{align}
\Minimize_{\boldsymbol{\beta}_n}\quad 
&\sum_{k=1}^{K} G(\alpha_{k,n}, \beta_{k,n})  - \mu^{-1} \Bigg[
\sum_{k=1}^{K} \lg(\beta_{k,n})  \nonumber \\
&+ \sum_{k=1}^{K} \lg(1-\beta_{k,n})
+ \lg\!\left( 1 - \sum_{k=1}^{K} \beta_{k,n} \right)
\Bigg],
\label{b_simple}
\end{align}
where $\mu$ is the penalty parameter. As $\mu$ approaches infinity, the solution to problem (\ref{b_simple}) converges to the solution of problem (\ref{Qb}). Problem (\ref{b_simple}) can be solved through the gradient descent method with complexity of $\mathcal{O}(Km)$, where $K$ is the number of variables and $m$ is the number of iteration~\cite{roy2024elements}. Compared with solving problem (\ref{Qb}) via SCA, the low-complexity method is more suitable for large-scale device scenarios.
}

\section{Numerical Results}
\label{Numerical_Results}
In this section, we perform numerical tests to validate the one-step convergence analysis and evaluate the performance of the proposed SP-FL. Consider the scenario with devices distributed in an area of $500$~m radius, with a central PS. The CIFAR-$10$ dataset are used for the FL performance evaluation. A convolutional neural network (CNN) with $60,000$ parameters is trained for classification, which includes two convolutional layers and three fully connected layers. The max pooling operation is applied after each convolutional layer, with ReLU as the activation function. In the independent and identically distributed (IID) data scenario, the entire data set is first shuffled and then divided into equal segments of $2000$ samples each, with each device assigned one segment. In the non independent and identically distributed (non-IID) scenario, the composition of classes owned by devices follows a Dirichlet distribution.

\rev{Unless otherwise specified, we consider the non-IID scenario and select the modulus vector of the global gradient from the previous iteration as reference for the compensation gradient.} The other parameters are set to: the number of devices, $K = 20$, the bandwidth, $B = 10$~MHz, \rev{the path loss exponent, $\zeta = 3$~\cite{8870236,10146443},} the noise power, $N_0 = -174$~dBm/Hz, the transmit power, $P = -4$~dBm, the number of quantization bits, $b = 3$, the transmission latency threshold, $\tau = 0.5$~s, the learning rate, $\eta = 0.05$, and the Dirichlet hyperparameter, $\alpha = 0.5$. We set $L = 1/\eta$ for the Lipschitz constant \cite{9210812}. In addition, the parameters $\delta$, which is an upper bound on the quantization error, is estimated by simulation experiments \cite{8664630}. 

For comparison, we consider the following four typical baselines.
\begin{itemize}
	\item Error-free: The quantized local gradients are transmitted to the PS without errors.
	\item Scheduling \cite{9337227}: \rev{The PS selects a fixed proportion of devices (75\%) with the highest channel gain magnitudes.}
	\item DDS \cite{10552192}: The bandwidth is uniformly allocated to each device, with the PS discarding erroneous gradients without requesting retransmission.
	\item One-bit \cite{9272666}: Only sign bits of gradients are transmitted, and any transmission error leads to direct discarding.
\end{itemize}

\vspace{-0.4cm}
\subsection{Convergence Behavior}
Fig. \ref{bound} presents a comparison between the exact values of the loss function and the upper bound derived in \textbf{Theorem~\ref{Theorem1}} for both IID and non-IID data scenarios. From this figure, we can observe that the theoretical bound closely aligns with the actual loss function, validating the reliability of one-step convergence analysis and the resource allocation strategy derived from the convergence bound. Moreover, the theoretical convergence bound is more relaxed in the non-IID case, because the non-IID nature of data distribution results in greater fluctuation in the variance between local and global gradients, thereby relaxing the upper bound $\epsilon_{k,n}$ and the Jensen's inequality (\ref{A2}) in Appendix \ref{proof_theorem_1}. 

\begin{figure}
    \centering
    \includegraphics[width=0.9\columnwidth]{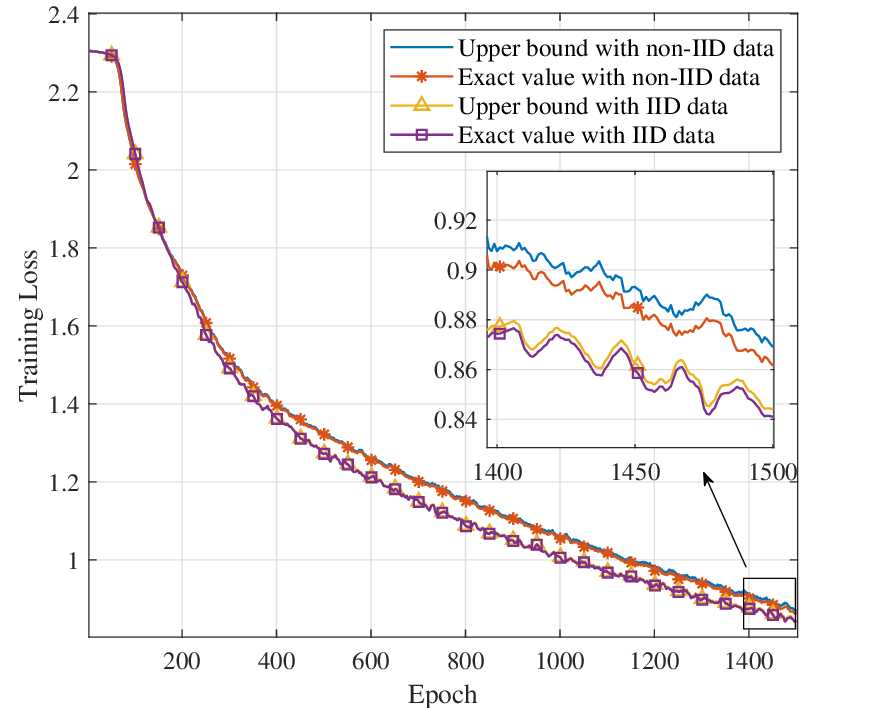}
    \caption{The upper bound and exact value of the loss function with IID and non-IID data.}
    \label{bound}
\end{figure}

\subsection{Performance Comparison}
\subsubsection{\rev{Performance comparison under varying non-IID levels}}
\rev{In Fig. \ref{non_iid}, we present the convergence performance of the proposed SP-FL and baseline methods under varying degrees of non-IID data distributions, where the device-level class composition follows a Dirichlet distribution with parameters $\alpha = 0.1$ and $\alpha = 0.01$~\cite{10318063}. It is observed that despite the added computational complexity, the proposed SP-FL achieves performance close to ideal error-free FL and consistently outperforms all baselines, yielding a 9.96\% performance improvement when $\alpha = 0.1$ and a 8.47\% improvement when $\alpha = 0.01$. These results demonstrate the robustness and effectiveness of the proposed SP-FL, especially in highly heterogeneous data settings.}
\begin{figure}
    \centering
    \includegraphics[width=0.9\columnwidth]{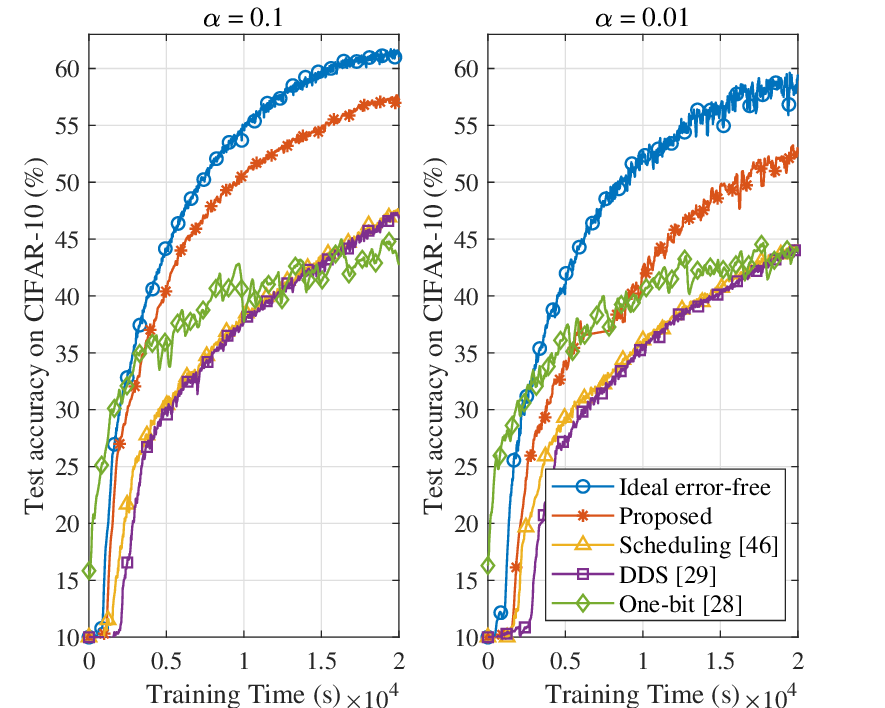}
    \caption{\rev{Performance comparison under varying non-IID levels.}}
    \label{non_iid}
\end{figure}
\subsubsection{\rev{Performance comparison with low-complexity optimization method}}
\rev{In Fig. \ref{convergence}, we compare the convergence performance of SP-FL with SCA and with the low-complexity method against baseline methods. With the device number $k=20$ and $k=30$, both SP-FL and the low-complexity variant achieve clear performance gains under the same training time. Specifically, when $k=20$, SP-FL with SCA achieves test accuracy and convergence time that are the closest to FL with error-free transmission, while SP-FL with low-complexity method achieves superior performance with $k=30$. Therefore, compared with the SCA-based design, the low-complexity method is a more suitable choice for large-scale systems with a large number of participating devices.}
\begin{figure}
    \centering
    \includegraphics[width=0.9\columnwidth]{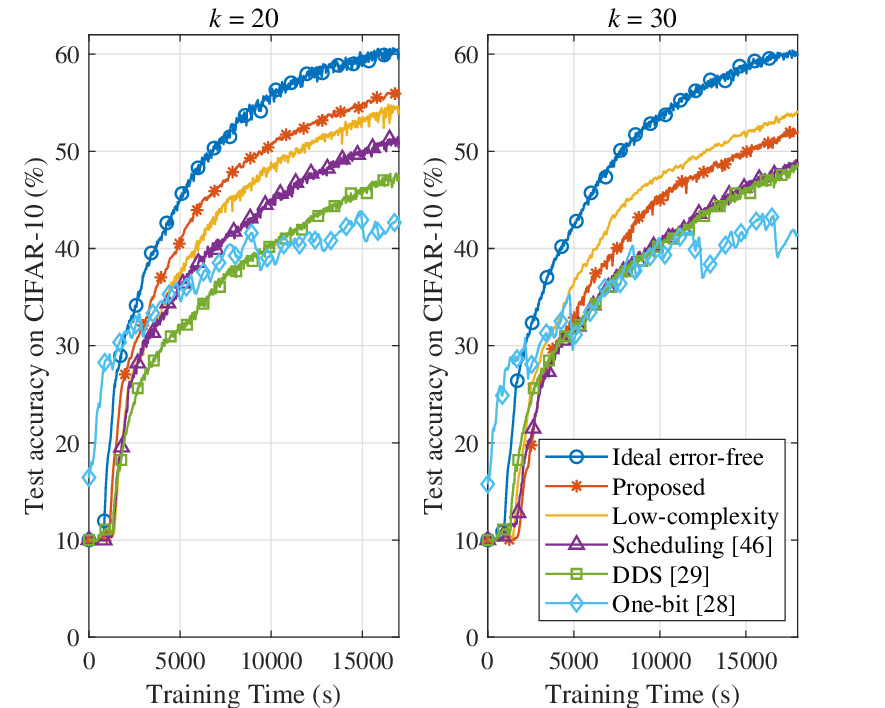}
    \caption{\rev{Performance comparison with low-complexity optimization method.}}
    \label{convergence}
\end{figure}
\rev{
\subsubsection{Performance comparison of different gradient compensation designs}
We further conduct experiments of other compensation methods, where the local gradient from the previous communication round of the corresponding device is employed as the reference, shown in Fig. \ref{compensation}. Experimental results demonstrate that the proposed SP-FL achieves consistent performance gains under different compensation designs, highlighting the effectiveness of the compensation strategy. We also observe that, compared with historical global gradients, using historical local gradients for compensation achieves improved performance. This is because historical local gradient can be more closely tied to the distribution of local data and learning trajectory~\cite{9756506}.}
\begin{figure}
    \centering
    \includegraphics[width=0.9\columnwidth]{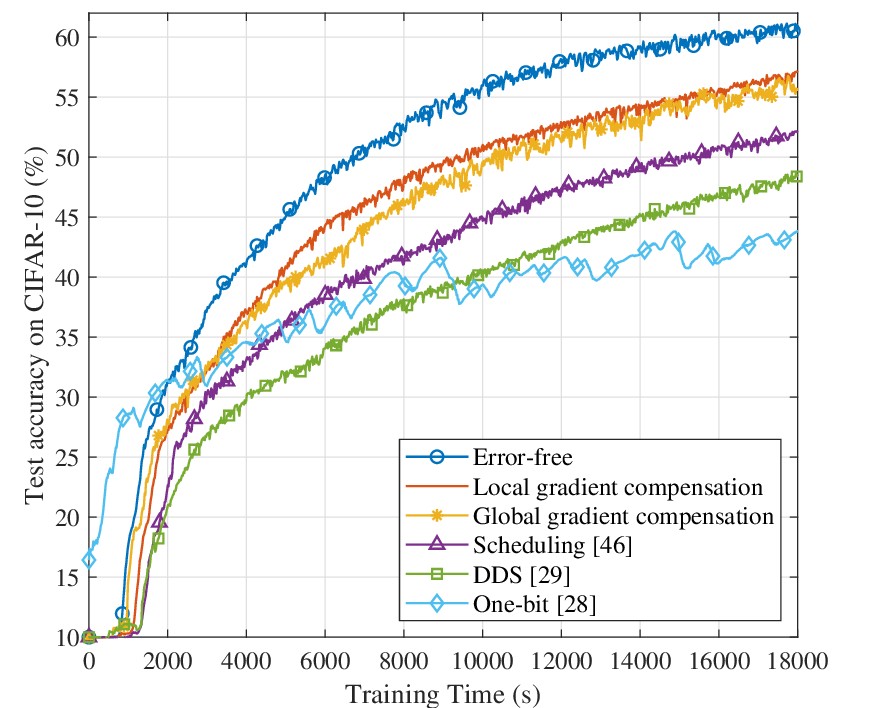}
    \caption{\rev{Performance comparison under different compensation designs.}}
    \label{compensation}
\end{figure}
\subsubsection{\rev{Performance evaluation of sign retransmission mechanisms}}
\rev{we conducted the following simulation to validate the feasibility of the sign retransmission mechanism in SP-FL, where erroneous sign packets are retransmitted to the server. As shown in Fig. \ref{sign_retrans}, although retransmitting sign packets introduces additional delay in each iteration, incorporating retransmission mechanism enables SP-FL to achieve improved convergence behavior and higher test accuracy compared with baseline methods. These results substantiate that enhancing the reliability of sign-packet transmission is crucial for strengthening the robustness and overall performance of SP-FL under practical communication constraints.}
\begin{figure}
    \centering
    \includegraphics[width=0.9\columnwidth]{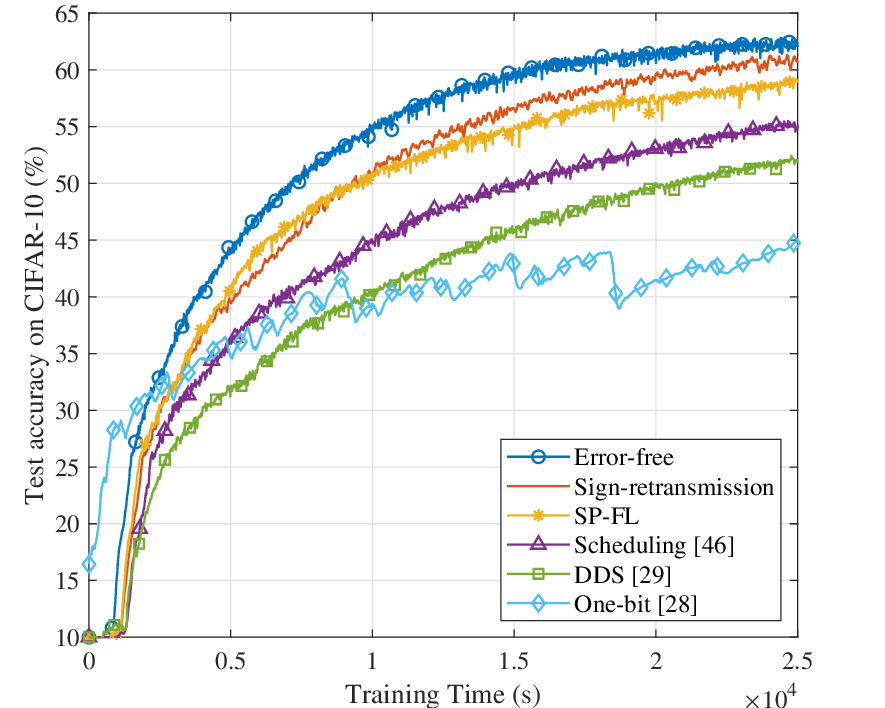}
    \caption{\rev{Performance comparison with retransmission designs.}}
    \label{sign_retrans}
\end{figure}

\subsection{Impact of Wireless factors}
\subsubsection{Impact of transmit power budget}
In Fig. \ref{cifar_P_comp_baseline}, we present the test accuracy of different FL methods versus the maximum transmit power per device. We observe that SP-FL consistently outperforms FL with scheduling, DDS, and one-bit methods, while achieving notable performance gains with limited power and near-error-free performance with sufficient power. Limited transmit power results in an increase in transmission errors, and SP-FL mitigates this effect by prioritizing the transmission of more important packets. Additionally, the one-bit method outperforms other baselines in low-power scenarios, highlighting the significance of signs in model updating. As transmit power increases, SP-FL further distinguishes itself by leveraging the sign-prioritized resource allocation strategy and fully utilizing the received sign packets, achieving superior performance. Moreover, the DDS method demonstrates superior performance compared to the scheduling method when sufficient transmit power is available. This finding underscores that enabling as many devices as possible to participate in global aggregation can significantly enhance system performance, thereby validating the rationale behind the SP-FL.

\begin{figure}
    \centering
    \includegraphics[width=0.9\columnwidth]{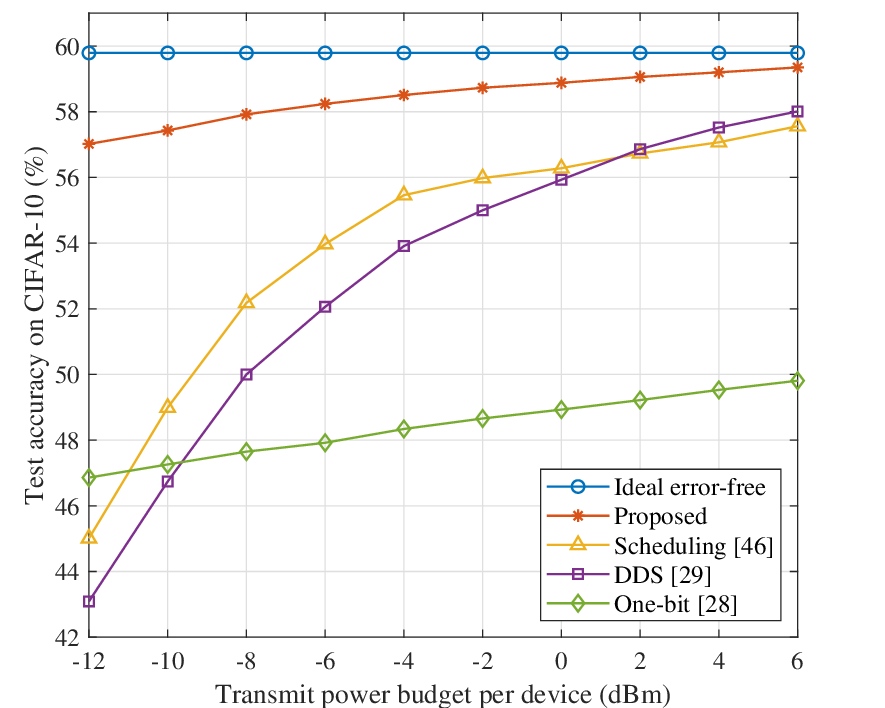}
    \caption{\rev{Test accuracy versus transmit power.}}
    \label{cifar_P_comp_baseline}
\end{figure}

\subsubsection{Impact of transmission latency thresholds}
In Fig. \ref{cifar_tau_comp}, we present the test accuracy of different FL methods versus transmission latency thresholds. It is observed that the proposed SP-FL shows performance gains under different latency thresholds. When the transmission latency threshold is stringent, SP-FL demonstrates superior performance. This suggests that under low-latency conditions, reliably transmitting the few but crucial sign bits improves efficiency. As the transmission latency threshold increases, the performance improvement follows a trend analogous to that observed with increasing transmit power in Fig. \ref{cifar_P_comp_baseline}. This is because the performance degradation at shorter latency thresholds also arises from limited wireless resources, which makes it challenging to support reliable transmission under strict latency requirements.

\begin{figure}
    \centering
    \includegraphics[width=0.9\columnwidth]{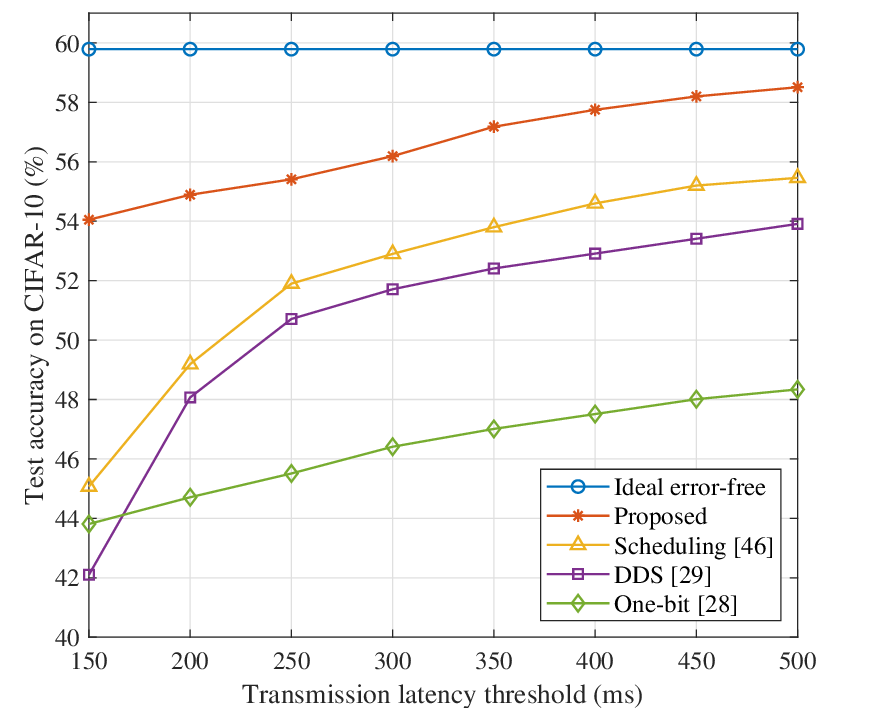}
    \caption{\rev{Test accuracy versus transmission latency thresholds.}}
    \label{cifar_tau_comp}
\end{figure}

\subsubsection{Impact of device number}
Fig. \ref{cifar_num_comp} presents the test accuracy achieved by various schemes as a function of the number of participating devices. As shown in this figure, the proposed SP-FL consistently demonstrates performance superiority across different numbers of devices. When the number of devices is limited, the FL system benefits from sufficient bandwidth resources but is constrained by a limited global dataset, resulting in performance degradation. Notably, SP-FL exhibits a more gradual decline in accuracy compared to FL with DDS and scheduling methods. As the number of devices increases, the expanding global dataset enhances FL performance. However, the emerging bandwidth limitations cause a reduction in accuracy. \rev{It is worth noting that, the test accuracy slightly decreases across different methods when $K$ grows beyond a moderate range, which confirms that the accuracy drop is not caused by the proposed SP-FL algorithm itself, but rather by the fundamental resource limitation inherent in wireless FL systems. Meanwhile, the scheduling policy maintains a fixed scheduling ratio, meaning that as the total number of devices increases, a proportionally larger subset is scheduled in each round. This inevitably introduces additional transmission errors due to limited per-device resources, leading to a mild degradation in accuracy.} Nevertheless, the one-bit method significantly reduces transmitted data, which mitigates resource constraints and achieves a steady performance improvement.

\begin{figure}
    \centering
    \includegraphics[width=0.9\columnwidth]{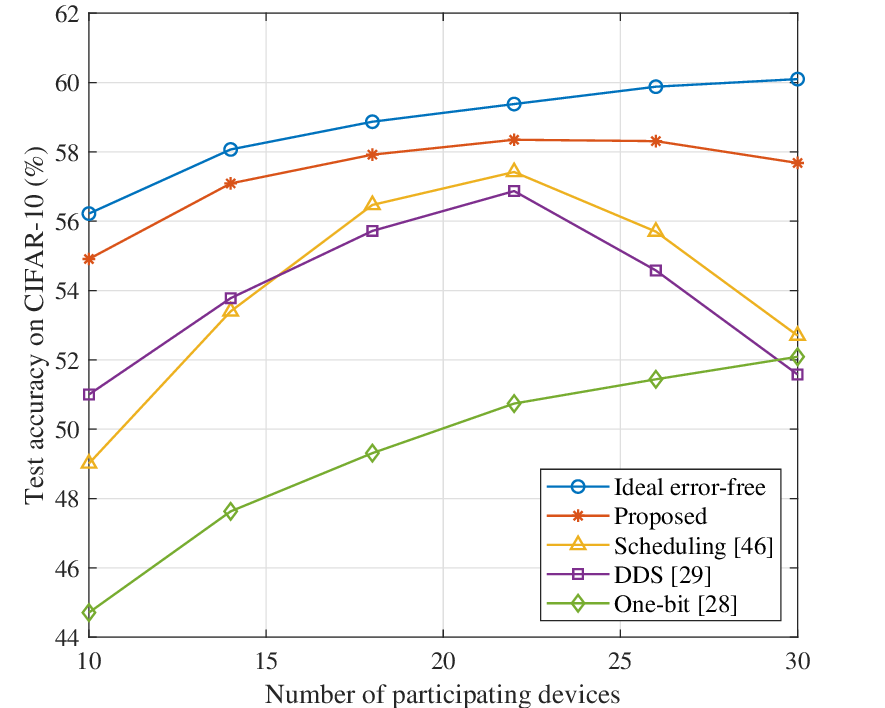}
    \caption{\rev{Test accuracy versus the number of participating devices.}}
    \label{cifar_num_comp}
\end{figure}

\subsubsection{Impact of quantization bits}
Fig. \ref{cifar_bit_comp} depicts the test accuracy as a function of the number of quantization bits under varying transmit power levels. The FL performance initially improves with increasing quantization precision but eventually declines, as higher quantization levels exacerbate transmission errors due to limited wireless resources. This results in a performance peak with the optimal number of quantization bits. Additionally, higher transmission power shifts the optimal number of bits upward, as increased wireless resources can accommodate the accurate transmission of more precise gradients. Consequently, optimizing the quantization bit level based on available wireless resources is essential for practical deployment. Given the discrete nature of quantization bit values, a low-complexity exhaustive search is sufficient to determine the optimal setting.

\begin{figure}
    \centering
    \includegraphics[width=0.9\columnwidth]{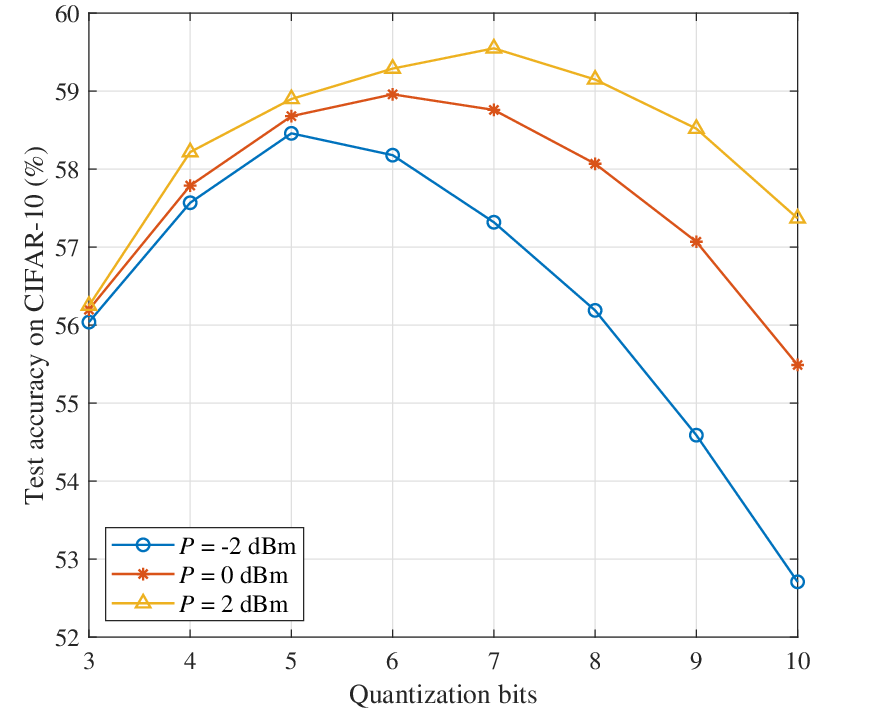}
    \caption{\rev{Test accuracy versus the number of quantization bits.}}
    \label{cifar_bit_comp}
\end{figure}

\vspace{-0.4cm}
\section{Conclusion}
\label{conclusion}
In this paper, we have proposed a sign-prioritized FL scheme for resource-constrained wireless FL systems, namely SP-FL, which prioritizes the transmission of important gradient information through tailored FL designs and preferential resource allocation. Considering the importance of gradient direction, we have introduced a sign-modulus decoupled transmission strategy and facilitated the reuse of correct sign packets with erroneous modulus packets. Furthermore, based on the one-step convergence analysis, we have developed a hierarchical resource allocation strategy that optimizes the allocation of system bandwidth among devices and transmit power among the sign and modulus packets. Numerical results have validated the superior performance of SP-FL under resource-constrained scenarios such as limited transmit power and high device density. In addition, SP-FL is shown to dynamically adapt to resource conditions to ensure superior performance.

\vspace{-0.3cm}
\appendices
\section{Proof of Theorem \ref{Theorem1}}
\label{proof_theorem_1}
To prove \textbf{Theorem~\ref{Theorem1}}, we first rewrite $F\left(\Tilde{\mathbf{w}}_{n+1}\right) - F\left(\Tilde{\mathbf{w}}_{n}\right)$ based on Assumption \ref{assumption2}, which is expressed as
\vspace{-0.2cm}
\begin{align}
F\!\left(\!\Tilde{\mathbf{w}}_{n+1}\!\right)\! -\! F\!\left(\Tilde{\mathbf{w}}_{n}\!\right) &\!\leq \!\nabla \!F\left(\!\Tilde{\mathbf{w}}_{n}\!\right)\!^{T} \!\left(\!\Tilde{\mathbf{w}}_{n+1}-\Tilde{\mathbf{w}}_{n}\!\right) \!+\!\frac{L}{2}\!\left\|\!\Tilde{\mathbf{w}}_{n+1}\!-\!\Tilde{\mathbf{w}}_{n}\!\right\|^{2} \nonumber\\
&=\underbrace{-\eta{\mathbf g}_{n}^{T}\hat{\mathbf g}_{n}}_{A_1}+\underbrace{\frac{L\eta^{2}}{2}\|\hat{\mathbf g}_{n}\|^{2}}_{A_2}.
\label{33_F}
\end{align}
Substituting (\ref{hat_g}) into (\ref{33_F}), we rewrite $A_1$ as
\begin{align}
A_{1}=&-\frac\eta K\sum_{k=1}^K  \mathbf{g}_n^T \left(\frac{C(\mathbf{g}_{k,n})\cdot s(\mathbf{g}_{k,n})\odot\hat{Q}_v(\mathbf{g}_{k,n})}{q_{k,n}}\right)  \nonumber\\
=\,&\frac\eta{2K}\sum_{k=1}^K\left\|\frac{C(\mathbf{g}_{k,n})\cdot s(\mathbf{g}_{k,n})\odot\hat{Q}_v(\mathbf{g}_{k,n})}{q_{k,n}}-\mathbf{g}_n\right\|^2 \nonumber\\
&\!-\!\frac\eta2\!\|\mathbf{g}_n\!\|^2\!-\!\frac\eta{2K}\!\sum_{k=1}^K\!\left\|\!\frac{C(\!\mathbf{g}_{k,n}\!)\!\cdot\! s(\!\mathbf{g}_{k,n}\!)\!\odot\!\hat{Q}_v(\!\mathbf{g}_{k,n}\!)}{q_{k,n}}\!\right\|^2\!. 
\label{A1}
\end{align}
Then, by exploiting the Jensen’s inequality, $A_2$ is bounded by
\begin{align}
A_{2} &= \frac{L \eta^{2}}{2}\left\|\sum_{k=1}^{K}\frac{C(\mathbf{g}_{k,n})\cdot s(\mathbf{g}_{k,n})\odot\hat{Q}_{v}(\mathbf g_{k,n})}{Kq_{k,n}}\right\|^{2} \nonumber \\
&\leq \frac{L\eta^{2}}{2K}{\sum_{k=1}^{K}\left\|\frac{C(\mathbf{g}_{k,n})\cdot s(\mathbf{g}_{k,n})\odot\hat{Q}_{v}(\mathbf g_{k,n})}{q_{k,n}}\right\|^{2}}.
\label{A2}
\end{align}
Combining the results in (\ref{A1}) and (\ref{A2}), the expectation of (\ref{33_F}) is expressed as
\begin{align}
&\mathbb E[F\left(\Tilde{\mathbf{w}}_{n+1}\right)] - F\left(\Tilde{\mathbf{w}}_{n}\right) \leq -\frac{\eta}{2}\|\mathbf g_{n}\|^{2} \nonumber\\
&-\frac{\eta(1-L\eta)}{2K}\underbrace{\sum_{k=1}^{K}\mathbb E\left[\left\|\frac{C(\mathbf{g}_{k,n})\cdot s(\mathbf{g}_{k,n})\odot\hat{Q}_{v}(\mathbf g_{k,n})}{q_{k,n}}\right\|^{2}\right]}_{B_1} \nonumber\\
&\!+\!\frac{\eta}{2K}\!\underbrace{\sum_{k=1}^{K}\mathbb E\!\left[\!\left\|\frac{C(\mathbf{g}_{k,n})\cdot s(\mathbf{g}_{k,n}))\odot\hat{Q}_{v}(\mathbf g_{k,n})}{q_{k,n}}
\!-\!\mathbf g_{n}\!\right\|^{2}\!\right]}_{B_2}.
\label{22_EF}
\end{align}
\vspace{-0.2cm}
For the term $B_1$, we have
\rev{
\begin{align}
B_{1} 
=  \sum_{k=1}^{K}\mathbb E\left[\left\|\frac{C(\mathbf{g}_{k,n})\cdot s(\mathbf{g}_{k,n})\odot\hat{Q}_{v}(\mathbf g_{k,n})}{q_{k,n}}\right\|^{2}\right].
\label{B1_1}
\end{align}
Due to $E\left[C(\mathbf{g}_{k,n})\right]=q_{k,n}$, and the correct transmission probability of $s(\mathbf{g}_{k,n})$ is $p_{k,n}$, we have
\begin{align}
\sum_{k=1}^{K}&\mathbb E\left[\left\|\frac{C(\mathbf{g}_{k,n})\cdot s(\mathbf{g}_{k,n})\odot\hat{Q}_{v}(\mathbf g_{k,n})}{q_{k,n}}\right\|^{2}\right] \nonumber\\
=& \sum_{k=1}^{K} \frac{1}{q_{k,n}} \mathbb E\left[\left\|{s(\mathbf{g}_{k,n})\odot\hat{Q}_{v}(\mathbf g_{k,n})}\right\|^{2}\right]\nonumber\\
 =& \sum_{k=1}^{K}\left(\frac{p_{k,n}}{q_{k,n}}\mathbb E\left[\|Q(\mathbf g_{k,n})\|^{2}\right] +\frac{1-p_{k,n}}{q_{k,n}}\|\bar{\mathbf g}\|^{2}\right).
 \label{B1_2}
\end{align}
According to \textbf{Lemma 2}, we have $E\left[Q(\mathbf g_{k,n})\right]=\mathbf g_{k,n}$. Therefore,
\begin{align}
\sum_{k=1}^{K}&\!\left(\!\frac{p_{k,n}}{q_{k,n}}\mathbb E\left[\|Q(\mathbf g_{k,n})\|^{2}\right]\! +\!\frac{1-p_{k,n}}{q_{k,n}}\|\bar{\mathbf g}\|^{2}\!\right) \nonumber\\
\!= \!&\sum_{k=1}^{K}\!\left(\!\frac{p_{k,n}}{q_{k\!,n}}\!\mathbb E\left[\!\|Q(\!\mathbf g_{k\!,n}\!)\!-\!\mathbf g_{k\!,n}\!\|^{2}\!\right]\!+\!\frac{p_{k\!,n}}{q_{k\!,n}}\|\mathbf g_{k\!,n}\|^{2} \!+\!\frac{1\!-\!p_{k\!,n}}{q_{k\!,n}}\!\|\bar{\mathbf g}\|^{2}\!\right).
 \label{B1_3}
\end{align}
Meanwhile, according to \textbf{Lemma 2}, we have $\mathbb E\left[\|Q(\mathbf g_{k,n})-\mathbf g_{k,n}\|^{2}\right]\leq\delta_{k,n}^2$. It gives
\begin{align}
\sum_{k=1}^{K}&\!\left(\!\frac{p_{k,n}}{q_{k\!,n}}\!\mathbb E\left[\!\|Q(\!\mathbf g_{k\!,n}\!)\!-\!\mathbf g_{k\!,n}\!\|^{2}\!\right]\!+\!\frac{p_{k\!,n}}{q_{k\!,n}}\|\mathbf g_{k\!,n}\|^{2} \!+\!\frac{1\!-\!p_{k\!,n}}{q_{k\!,n}}\!\|\bar{\mathbf g}\|^{2}\!\right)\nonumber\\
{\leq} & \sum_{k=1}^{K}\!\left(\!\frac{p_{k,n}}{q_{k,n}}\delta_{k,n}^2+\frac{p_{k,n}}{q_{k,n}}\|\mathbf g_{k,n}\|^{2}+\frac{1-p_{k,n}}{q_{k,n}}\|\bar{\mathbf g}\|^{2}\!\right)\!.
\label{B1_4}
\end{align}
Therefore, we can derive that
\begin{align}
B_{1} {\leq}\!\sum_{k=1}^{K}\!\left(\!\frac{p_{k,n}}{q_{k,n}}\delta_{k,n}^2+\frac{p_{k,n}}{q_{k,n}}\|\mathbf g_{k,n}\|^{2}+\frac{1-p_{k,n}}{q_{k,n}}\|\bar{\mathbf g}\|^{2}\!\right)\!.
\label{B1}
\end{align}
}

Next, we express $B_2$ as (\ref{B2}), where (a) stems from the fact that 
\begin{equation}
\mathbb E \!\left[\!\frac{C\!(\!\mathbf{g}_{k\!,\!n}\!)\!\cdot\! s(\!\mathbf{g}_{k\!,\!n}\!)\!\odot\!\hat{Q}_{v}\!(\!\mathbf g_{k,n}\!)}{q_{k,n}}\!\right]
\!=\!p_{k\!,n} \mathbf{g}_{k\!,\!n}\!+\!(\!1\!-\!p_{k\!,\!n}\!)s(\!\mathbf{g}_{k\!,\!n}\!)\!\odot\!\bar{\mathbf{g}}. 
\end{equation}

\begin{figure*}[!h]
\centering
\begin{align}
B_{2}\!=\!
& {\sum_{k=1}^{K}\!\mathbb{E}\!\left[\!\left\|\!\frac{C(\mathbf{g}_{k,n})\!\cdot s(\mathbf{g}_{k,n})\!\odot\!\hat{Q}_{v}(\!\mathbf{g}_{k,n}\!)}{q_{k,n}}\!-\!p_{k,n} \mathbf{g}_{k,n}\!-\!(\!1\!-\!p_{k,n}\!) s(\!\mathbf{g}_{k,n}\!)\!\odot\!\bar{\mathbf{g}}\right\|^{2}\right]} 
\!+\!{\sum_{k=1}^{K}\!\left\|p_{k,n} \mathbf{g}_{k,n}\!+\!(\!1\!-\!p_{k,n}) s(\mathbf{g}_{k,n})\!\odot\!\bar{\mathbf{g}}\!-\!\mathbf{g}_{n}\right\|^{2}}
 \nonumber\\
=&\underbrace{\sum_{k=1}^{K}\left\|(1-p_{k,n})(s(\mathbf{g}_{k,n})\odot\bar{\mathbf{g}}- \mathbf{g}_{k,n})+\mathbf{g}_{k,n}-\mathbf g_{n}\right\|^{2}}_{C_1} +\underbrace{\sum_{k=1}^{K} (1-q_{k,n})\left\|p_{k,n} \mathbf{g}_{k,n}+(1-p_{k,n})s(\mathbf{g}_{k,n})\odot\bar{\mathbf{g}}\right\|^{2}}_{C_2} \nonumber\\
&+\underbrace{\sum_{k=1}^{K} q_{k,n}\mathbb E\left[\left\|\frac{s(\mathbf{g}_{k,n})\odot\hat{Q}_{v}(\mathbf g_{k,n})}{q_{k,n}}-p_{k,n} \mathbf{g}_{k,n}  -(1-p_{k,n})s(\mathbf{g}_{k,n})\odot\bar{\mathbf{g}}\right\|^{2} \right]}_{C_3}
\label{B2}
\end{align}
\vspace{-0.6cm}
\hrulefill
\end{figure*}

We further bound $C_1$ in (\ref{B2}) by 
\vspace{-0.2cm}
\begin{align}
C_1 
=\,&\sum_{k=1}^{K}\left\|(1-p_{k,n})(s(\mathbf{g}_{k,n})\odot\bar{\mathbf{g}}- \mathbf{g}_{k,n})+\mathbf{g}_{k,n}-\mathbf g_{n}\right\|^{2} \nonumber\\
\leq\, & 2\!\sum_{k=1}^{K}\!(\!1-p_{k,n}\!)^2 \left\|s(\mathbf{g}_{k,n})\odot\bar{\mathbf{g}}\!-\! \mathbf{g}_{k,n}\right\|^{2} \!+\! 2\sum_{k=1}^{K} \left\|\mathbf{g}_{k,n}\!-\!\mathbf g_{n}\right\|^{2}  \nonumber\\
\overset{(\mathrm{a})}{\operatorname*{\leq}} \,& 2\sum_{k=1}^{K} \left(\left\|\bar{\mathbf{g}}\right\|^{2}+ \left\|\mathbf{g}_{k,n}\right\|^{2}-2\upsilon_{k,n}+\epsilon_{k,n}^2\right) \nonumber\\
&- 4\sum_{k=1}^{K} p_{k,n} \left(\left\|\bar{\mathbf{g}}\right\|^{2}+ \left\|\mathbf{g}_{k,n}\right\|^{2}-2\upsilon_{k,n}\right) \nonumber\\
&+ 2\sum_{k=1}^{K} p_{k,n}^2 \left(\left\|\bar{\mathbf{g}}\right\|^{2}+ \left\|\mathbf{g}_{k,n}\right\|^{2}-2\upsilon_{k,n}\right),
\label{C1}
\end{align}
where (a) is from Assumption \ref{assumption3}, and $\upsilon_{k,n}$ is defined as $\upsilon_{k,n} \triangleq \langle \mathbf{g}_{k,n}, s(\mathbf{g}_{k,n})\odot\bar{\mathbf{g}}\rangle \geq 0$. 
Then, the second term on the right-hand side of (\ref{B2}), $C_2$, is rewritten as
\vspace{-0.2cm}
\begin{align}
C_2 =&\sum_{k=1}^K(1-q_{k,n})\left\|p_{k,n}\mathbf{g}_{k,n}+(1-p_{k,n})s(\mathbf{g}_{k,n})\odot\bar{\mathbf{g}}\right\|^2 \nonumber\\
=&\sum_{k=1}^K \left\|\bar{\mathbf{g}}\right\|^2 + 2\sum_{k=1}^K p_{k,n}\left(\upsilon_{k,n}-\left\|\bar{\mathbf{g}}\right\|^2\right) \nonumber\\
&-\sum_{k=1}^K q_{k,n}\left\|\bar{\mathbf{g}}\right\|^2+2\sum_{k=1}^K p_{k,n}q_{k,n}\left(\left\|\bar{\mathbf{g}}\right\|^2-\upsilon_{k,n}\right)\nonumber\\
&+\sum_{k=1}^K p_{k,n}^2 \left(\left\|\mathbf{g}_{k.n}\right\|^2+\left\|\bar{\mathbf{g}}\right\|^2-2\upsilon_{k,n}\right) \nonumber\\
&+\sum_{k=1}^K p_{k,n}^2q_{k,n} \left(2\upsilon_{k,n}-\left\|\mathbf{g}_{k,n}\right\|^2-\left\|\bar{\mathbf{g}}\right\|^2\right).
\label{C2}
\vspace{-0.2cm}
\end{align}
For the third term on the right-hand side of (\ref{B2}), $C_3$, it is bounded by (\ref{C3}) at the top of the next page, where inequality (a) is due to \textbf{Lemma~\ref{lemma}}.

\begin{figure*}[!h]
\centering
\begin{align}
C_3
\!=&\!\sum_{k=1}^{K}p_{k\!,n}q_{k\!,n}\mathbb E\!\left[\!\left\|\!\frac{Q(\!\mathbf g_{k\!,n}\!)}{q_{k,n}}
\!-\!p_{k\!,n}\mathbf{g}_{k\!,n}
\!-\!(\!1\!-\!p_{k,n}\!)s(\mathbf{g}_{k\!,n}\!)\!\odot\!\bar{\mathbf{g}}\!\right\|^{2} \!\right]
\!+\!\sum_{k=1}^{K}\!(1\!-\!p_{k\!,n}\!)q_{k,n}\!\left\|\frac{s(\!\mathbf{g}_{k\!,n}\!)\!\odot\!\bar{\mathbf{g}}}{q_{k\!,n}}
\!-\!p_{k\!,n} \mathbf{g}_{k\!,n}
\!-\!(\!1\!-\!p_{k\!,n}\!)s(\!\mathbf{g}_{k\!,n}\!)\!\odot\!\bar{\mathbf{g}}\right\|^{2} \nonumber\\
=&\sum_{k=1}^{K}\frac{p_{k,n}}{q_{k,n}}\mathbb E\left[\left\|Q(\mathbf g_{k,n})-\mathbf g_{k,n}\right\|^{2}\right]
+\sum_{k=1}^{K}p_{k,n}q_{k,n}\left\|\frac{1-p_{k,n}q_{k,n}}{q_{k,n}}\mathbf{g}_{k,n}
-(1-p_{k,n})s(\mathbf{g}_{k,n})\odot\bar{\mathbf{g}}\right\|^{2} \nonumber\\
&+\sum_{k=1}^{K} (1-p_{k,n})q_{k,n}\left\|\frac{1-(1-p_{k,n})q_{k,n}}{q_{k,n}}s(\mathbf{g}_{k,n})\odot\bar{\mathbf{g}}-p_{k,n} \mathbf{g}_{k,n}\right\|^{2}\nonumber\\
\overset{(\mathrm{a})}{\operatorname*{\leq}} & \!\underbrace{\sum_{k=1}^{K} \!p_{k\!,\!n}\!q_{k\!,\!n}\!\left\|\!\frac{1\!-\!p_{k\!,\!n}q_{k\!,\!n}}{q_{k\!,\!n}}\!\mathbf{g}_{k\!,\!n}\!-\!(\!1\!-\!p_{k\!,\!n}\!)\!s(\!\mathbf{g}_{k\!,\!n}\!)\!\odot\!\bar{\mathbf{g}}\!\right\|^{2}}_{E_1}
\!+\!\underbrace{\sum_{k=1}^{K} (\!1\!-\!p_{k\!,n}\!)q_{k\!,\!n}\!\left\|\!\frac{1\!-\!(\!1\!-\!p_{k\!,\!n}\!)q_{k\!,\!n}}{q_{k\!,\!n}}\!s(\!\mathbf{g}_{k\!,\!n}\!)\!\odot\!\bar{\mathbf{g}}
\!-\!p_{k\!,\!n} \mathbf{g}_{k\!,\!n}\!\right\|^{2}}_{E_2} 
\!+\!\sum_{k=1}^{K}\frac{p_{k\!,\!n}}{q_{k\!,\!n}}\!\delta_{k\!,\!n}^{2}
\label{C3}
\end{align}
\vspace{-0.6cm}
\hrulefill
\end{figure*}

We further rewrite $E_1$ on the right-hand side of (\ref{C3}) as
\vspace{-0.2cm}
\begin{align}
E_{1} 
= & \sum_{k=1}^{K} \frac{p_{k,n}}{q_{k,n}}\left\|\mathbf{g}_{k,n} \right\|^{2} -2\sum_{k=1}^{K} p_{k,n}\upsilon_{k,n} +\sum_{k=1}^{K} p_{k,n}q_{k,n}\left\|\bar{\mathbf{g}} \right\|^{2}\nonumber\\
&\!+\!2\!\sum_{k=1}^{K}\!p_{k,n}^2 \!\left(\!\upsilon_{k,n}\!-\!\left\|\mathbf{g}_{k,n} \right\|^{2} \!\right) \!+\!\sum_{k=1}^{K}\!p_{k,n}^2 q_{k,n}\!\left(\!2\upsilon_{k,n} \!-\!2\left\|\bar{\mathbf{g}}\right\|^{2}\!\right)\nonumber\\
&+\sum_{k=1}^{K}p_{k,n}^3 q_{k,n}\left(\left\|\mathbf{g}_{k,n} \right\|^{2}+\left\|\bar{\mathbf{g}} \right\|^{2}-2\upsilon_{k,n}\right).
\label{E1}
\end{align} 
Similarly, the term $E_2$ is expressed as
\vspace{-0.2cm}
\begin{align}
E_{2} 
=&\!\sum_{k=1}^{K} \!(\!1\!-\!p_{k,n}\!)q_{k,n}\!\left\|\!\frac{1\!-\!(\!1\!-\!p_{k,n}\!)q_{k,n}}{q_{k,n}}\!s(\!\mathbf{g}_{k,n}\!)\!\odot\!\bar{\mathbf{g}}\!-\! p_{k,n} \mathbf{g}_{k,n}\!\right\|^{2} \nonumber\\
=&-2K\left\|\bar{\mathbf{g}} \right\|^{2}+ \sum_{k=1}^{K} \frac{1-p_{k,n}}{q_{k,n}}\left\|\bar{\mathbf{g}} \right\|^{2}+\sum_{k=1}^{K}q_{k,n}\left\|\bar{\mathbf{g}} \right\|^{2} \nonumber\\
&+\!2\!\sum_{k=1}^{K} \!p_{k,n}\!\left(2\left\|\bar{\mathbf{g}} \right\|^{2}\!-\!\upsilon_{k,n} \right)\!+\!2\!\sum_{k=1}^{K} \!p_{k,n}^2 \left(\upsilon_{k,n}\!-\!\left\|\bar{\mathbf{g}} \right\|^{2} \right) \nonumber\\
&+\sum_{k=1}^{K}p_{k,n}q_{k,n}\left(2\upsilon_{k,n}-3\left\|\bar{\mathbf{g}} \right\|^{2} \right) \nonumber\\
&+\sum_{k=1}^{K} p_{k,n}^2q_{k,n}\left(\left\|{\mathbf{g}_{k,n}} \right\|^{2}-4\upsilon_{k,n}+3\left\|\bar{\mathbf{g}} \right\|^{2} \right) \nonumber\\
&+\sum_{k=1}^{K} p_{k,n}^3q_{k,n}\left(2\upsilon_{k,n}-\left\|{\mathbf{g}_{k,n}} \right\|^{2}-\left\|\bar{\mathbf{g}} \right\|^{2} \right).
\label{E2}
\end{align} 

Combining all the results in (\ref{22_EF})-(\ref{E2}), it yields
\vspace{-0.2cm}
\begin{align}
\mathbb E&[F\left(\Tilde{\mathbf{w}}_{n+1}\right)] - F\left(\Tilde{\mathbf{w}}_{n}\right) \nonumber\\
 \leq&-\frac{\eta}{2}\|\mathbf g_{n}\|^{2}+\frac{\eta}{2}\|\bar{\mathbf g}\|^{2}+\frac{\eta}{K} \sum_{k=1}^{K} (\left\|\mathbf{g}_{k,n}\right\|^{2}+\epsilon_{k,n}^2-2\upsilon_{k,n})  \nonumber\\
&+\frac{\eta}{K} \sum_{k=1}^{K}p_{k,n}\left(-\left\|\bar{\mathbf{g}}\right\|^{2}-2 \left\|\mathbf{g}_{k,n}\right\|^{2}+3\upsilon_{k,n} \right) \nonumber\\
&+\frac{\eta}{2K} \sum_{k=1}^{K}p_{k,n}^2\left(\left\|\bar{\mathbf{g}}\right\|^{2}+\left\|\mathbf{g}_{k,n}\right\|^{2}-2\upsilon_{k,n} \right) \nonumber\\
&\!+\!\frac{L\eta^2}{2\!K} \!\sum_{k=1}^{K}\!\frac{p_{k\!,\!n}}{q_{k\!,\!n}}\!(\!\delta^2_{k\!,\!n}\!+\!\left\|\!\mathbf g_{k\!,\!n}\!\right\|^{2}\!-\!\left\|\!\bar{\mathbf g}\!\right\|\!^{2}\!)
\!+\!\frac{L\eta^2}{2\!K} \!\sum_{k=1}^{K}\!\frac{1}{q_{k\!,\!n}}\!\left\|\bar{\mathbf g}\!\right\|^{2}\!.
\end{align}
The proof is complete.

\section{Proof of Lemma \ref{power_allocation_lemma}}
\label{proof_lemma}
To prove the \textbf{Lemma~\ref{power_allocation_lemma}}, we first prove 
\vspace{-0.2cm}
\begin{equation}
\lim_{\alpha_{k,n} \to 0^+} G^{\prime}(\alpha_{k,n}) < 0,
\vspace{-0.3cm}
\end{equation}
where $G^{\prime}(\alpha_{k,n})$ is the first-order derivative of the objective function in (\ref{Q1}) with respect to $\alpha_{k,n}$ given by (\ref{first_derivative}). According to (\ref{H_s}) and (\ref{H_v}), we have $H_s(\beta_{k,n}) < 0$ and $H_v(\beta_{k,n}) < 0$. Based on these, we can derive that
\begin{align}
&C_{k,n}\exp\left(\frac{H_v(\beta_{k,n})}{1-\alpha_{k,n}}-\frac{H_s(\beta_{k,n})}{\alpha_{k,n}}\right){H_s(\beta_{k,n})} \nonumber\\
&+D_{k,n}\exp\left(-\frac{H_s(\beta_{k,n})}{\alpha_{k,n}}\right){H_s(\beta_{k,n})} < 0.
\end{align}
Hence, we have $\lim_{\alpha_{k,n} \to 0^+} G^{\prime}(\alpha_{k,n}) < 0$.

\begin{figure*}
\begin{align}
G^{\prime}(\alpha_{k,n})
=&A_{k,n}\exp\left(\frac{H_v(\beta_{k,n})}{1-\alpha_{k,n}}\right)\frac{H_v(\beta_{k,n})}{(1-\alpha_{k,n})^2}
+B_{k,n}\exp\left(\frac{2H_v(\beta_{k,n})}{1-\alpha_{k,n}}\right)\frac{2H_v(\beta_{k,n})}{(1-\alpha_{k,n})^2} \nonumber\\
&+C_{k,n}\exp\left(\frac{H_v(\beta_{k,n})}{1-\alpha_{k,n}}-\frac{H_s(\beta_{k,n})}{\alpha_{k,n}}\right)\left(\frac{H_v(\beta_{k,n})}{(1-\alpha_{k,n})^2} +\frac{H_s(\beta_{k,n})}{\alpha_{k,n}^2} \right)
+D_{k,n}\exp\left(-\frac{H_s(\beta_{k,n})}{\alpha_{k,n}}\right)\frac{H_s(\beta_{k,n})}{\alpha_{k,n}^2}
\label{first_derivative}
\end{align}
\hrulefill
\end{figure*}

Therefore, the solution $\alpha_{k,n}^*$ can be obtained case-by-case in the following.

(1) If there exists $0 < x_1<\dots< x_i< 1$, which satisfies $\forall 1\leq j \leq i$, $G^{\prime}(x_j)=0$, $G(\alpha_{k,n},\beta_{k,n})$ takes the minimal value at $\alpha_{k,n}=\alpha_{k,n}^*$, where
\begin{equation}
\alpha_{k,n}^* = \mathop{\arg\min}\limits_{\alpha_{k,n}\in \{x_1,\dots,x_i,1\}} G(\alpha_{k,n},\beta_{k,n}).
\end{equation}

(2) If $G^{\prime}(\alpha_{k,n}) = 0$ contains no solution $x$ between $(0,1)$, we can conclude that $G^{\prime}(\alpha_{k,n}) < 0$ for $\forall \alpha_{k,n} \in (0,1)$. Then $G(\alpha_{k,n},\beta_{k,n})$ takes a minimal value at $\alpha_{k,n}^* = 1$.

By combining (1) and (2), we complete the proof.


\bibliographystyle{IEEEtran}  
\bibliography{IEEEabrv,ref} 

\begin{thebibliography}{10}
\providecommand{\url}[1]{#1}
\csname url@samestyle\endcsname
\providecommand{\newblock}{\relax}
\providecommand{\bibinfo}[2]{#2}
\providecommand{\BIBentrySTDinterwordspacing}{\spaceskip=0pt\relax}
\providecommand{\BIBentryALTinterwordstretchfactor}{4}
\providecommand{\BIBentryALTinterwordspacing}{\spaceskip=\fontdimen2\font plus
\BIBentryALTinterwordstretchfactor\fontdimen3\font minus \fontdimen4\font\relax}
\providecommand{\BIBforeignlanguage}[2]{{%
\expandafter\ifx\csname l@#1\endcsname\relax
\typeout{** WARNING: IEEEtran.bst: No hyphenation pattern has been}%
\typeout{** loaded for the language `#1'. Using the pattern for}%
\typeout{** the default language instead.}%
\else
\language=\csname l@#1\endcsname
\fi
#2}}
\providecommand{\BIBdecl}{\relax}
\BIBdecl

\bibitem{11162025}
Y.~Yue, J.~Yao, J.~Xu, W.~Xu, Z.~Yang, and C.~Yuen, ``Priority-aware transmission for federated learning over wireless networks,'' in \emph{Proc. IEEE Int. Conf. Commun. (ICC)}, Montreal, Canada, Jun. 2025, pp. 1924--1929.

\bibitem{10604756}
W.~Xu, J.~Wu, S.~Jin, X.~You, and Z.~Lu, ``Disentangled representation learning empowered {CSI} feedback using implicit channel reciprocity in {FDD} massive {MIMO},'' \emph{IEEE Trans. Wireless Commun.}, vol.~23, no.~10, pp. 15\,169--15\,184, Oct. 2024.

\bibitem{10183789}
Y.~Shi, Y.~Zhou, D.~Wen, Y.~Wu, C.~Jiang, and K.~B. Letaief, ``Task-oriented communications for {6G}: Vision, principles, and technologies,'' \emph{IEEE Wireless Commun.}, vol.~30, no.~3, pp. 78--85, Jun. 2023.

\bibitem{11016266}
J.~Yao, W.~Xu, G.~Zhu, K.~Huang, and S.~Cui, ``Energy-efficient edge inference in integrated sensing, communication, and computation networks,'' \emph{IEEE J. Sel. Areas Commun.}, vol.~43, no.~10, pp. 3580--3595, Oct. 2025.

\bibitem{xu2023toward}
W.~Xu, Y.~Huang, W.~Wang, F.~Zhu, and X.~Ji, ``Toward ubiquitous and intelligent {6G} networks: From architecture to technology,'' \emph{Sci. China Inf. Sci.}, vol.~66, no.~3, pp. 130\,300:1--2, Mar. 2023.

\bibitem{10024766}
W.~Xu \emph{et~al.}, ``Edge learning for {B5G} networks with distributed signal processing: Semantic communication, edge computing, and wireless sensing,'' \emph{IEEE J. Sel. Top. Signal Process.}, vol.~17, no.~1, pp. 9--39, Jan. 2023.

\bibitem{10678839}
Y.~Shi, Y.~Yang, and Y.~Wu, ``Federated edge learning with differential privacy: An active reconfigurable intelligent surface approach,'' \emph{IEEE Trans. Wireless Commun.}, vol.~23, no.~11, pp. 17\,368--17\,383, Nov. 2024.

\bibitem{10064038}
W.~Ni, J.~Zheng, and H.~Tian, ``Semi-federated learning for collaborative intelligence in massive {IoT} networks,'' \emph{IEEE Internet Things J.}, vol.~10, no.~13, pp. 11\,942--11\,943, Jul. 2023.

\bibitem{10857353}
W.~Shi, J.~Yao, W.~Xu, J.~Xu, X.~You, Y.~C. Eldar, and C.~Zhao, ``Combating interference for over-the-air federated learning: A statistical approach via {RIS},'' \emph{IEEE Trans. Signal Process.}, vol.~73, pp. 936--953, 2025.

\bibitem{11180854}
Z.~Wang, Y.~Shi, Y.~Zhou, J.~Zhu, and K.~B. Letaief, ``Edge large {AI} models: Revolutionizing {6G} networks,'' \emph{IEEE Commun. Mag.}, vol.~63, no.~10, pp. 36--42, Oct. 2025.

\bibitem{11037631}
Z.~Wang, Y.~Shi, and K.~B. Letaief, ``Edge large {AI} models: Collaborative deployment and {IoT} applications,'' \emph{IEEE Internet Things Mag.}, vol.~8, no.~6, pp. 42--49, Nov. 2025.

\bibitem{10605604}
Y.~Shi, L.~Zeng, J.~Zhu, Y.~Zhou, C.~Jiang, and K.~B. Letaief, ``Satellite federated edge learning: Architecture design and convergence analysis,'' \emph{IEEE Trans. Wireless Commun.}, vol.~23, no.~10, pp. 15\,212--15\,229, Oct. 2024.

\bibitem{10767214}
Y.~Liang, Q.~Chen, G.~Zhu, H.~Jiang, Y.~C. Eldar, and S.~Cui, ``Communication-and-energy efficient over-the-air federated learning,'' \emph{IEEE Trans. Wireless Commun.}, vol.~24, no.~1, pp. 767--782, 2025.

\bibitem{jiang2025towards}
S.~Jiang, H.~Yang, Q.~Xie, C.~Ma, S.~Wang, Z.~Liu, T.~Xiang, and G.~Xing, ``Towards compute-efficient byzantine-robust federated learning with fully homomorphic encryption,'' \emph{Nature Machine Intelligence}, pp. 1--12, 2025.

\bibitem{10032291}
Z.~Wang, Y.~Zhou, Y.~Zou, Q.~An, Y.~Shi, and M.~Bennis, ``A graph neural network learning approach to optimize {RIS}-assisted federated learning,'' \emph{IEEE Trans. Wireless Commun.}, vol.~22, no.~9, pp. 6092--6106, Sept. 2023.

\bibitem{9843892}
Y.~Zou, Z.~Wang, X.~Chen, H.~Zhou, and Y.~Zhou, ``Knowledge-guided learning for transceiver design in over-the-air federated learning,'' \emph{IEEE Trans. Wireless Commun.}, vol.~22, no.~1, pp. 270--285, Jan. 2023.

\bibitem{10261509}
J.~Yao, Z.~Yang, W.~Xu, D.~Niyato, and X.~You, ``Imperfect {CSI}: A key factor of uncertainty to over-the-air federated learning,'' \emph{IEEE Wireless Commun. Lett.}, vol.~12, no.~12, pp. 2273--2277, Dec. 2023.

\bibitem{9264742}
Z.~Yang, M.~Chen, W.~Saad, C.~S. Hong, and M.~Shikh-Bahaei, ``Energy efficient federated learning over wireless communication networks,'' \emph{IEEE Trans. Wireless Commun.}, vol.~20, no.~3, pp. 1935--1949, Mar. 2021.

\bibitem{10142015}
J.~Yao, Z.~Yang, W.~Xu, M.~Chen, and D.~Niyato, ``Gomore: Global model reuse for resource-constrained wireless federated learning,'' \emph{IEEE Wireless Commun. Lett.}, vol.~12, no.~9, pp. 1543--1547, Sept. 2023.

\bibitem{10185584}
Y.~Mao \emph{et~al.}, ``{SAFARI}: Sparsity-enabled federated learning with limited and unreliable communications,'' \emph{IEEE Trans. Mobile Comput.}, vol.~23, no.~5, pp. 4819--4831, May 2024.

\bibitem{9716792}
H.~Ye, L.~Liang, and G.~Y. Li, ``Decentralized federated learning with unreliable communications,'' \emph{IEEE J. Sel. Top. Signal Process.}, vol.~16, no.~3, pp. 487--500, Apr. 2022.

\bibitem{9210812}
M.~Chen \emph{et~al.}, ``A joint learning and communications framework for federated learning over wireless networks,'' \emph{IEEE Trans. Wireless Commun.}, vol.~20, no.~1, pp. 269--283, Jan. 2021.

\bibitem{9611373}
Y.~Wang, Y.~Xu, Q.~Shi, and T.-H. Chang, ``Quantized federated learning under transmission delay and outage constraints,'' \emph{IEEE J. Sel. Areas Commun.}, vol.~40, no.~1, pp. 323--341, Jan. 2022.

\bibitem{10368103}
X.~Hou \emph{et~al.}, ``Efficient federated learning for metaverse via dynamic user selection, gradient quantization and resource allocation,'' \emph{IEEE J. Sel. Areas Commun.}, vol.~42, no.~4, pp. 850--866, Apr. 2024.

\bibitem{10038617}
J.~Du \emph{et~al.}, ``Gradient and channel aware dynamic scheduling for over-the-air computation in federated edge learning systems,'' \emph{IEEE J. Sel. Areas Commun.}, vol.~41, no.~4, pp. 1035--1050, Apr. 2023.

\bibitem{10145043}
Z.~Chen, W.~Yi, and A.~Nallanathan, ``Exploring representativity in device scheduling for wireless federated learning,'' \emph{IEEE Trans. Wireless Commun.}, vol.~23, no.~1, pp. 720--735, Jan. 2024.

\bibitem{10659225}
S.~Liu, Y.~Shen, J.~Yuan, C.~Wu, and R.~Yin, ``Storage-aware joint user scheduling and bandwidth allocation for federated edge learning,'' \emph{IEEE Trans. Cogn. Commun. Netw.}, vol.~11, no.~1, pp. 581--593, 2025.

\bibitem{9272666}
G.~Zhu, Y.~Du, D.~Gündüz, and K.~Huang, ``One-bit over-the-air aggregation for communication-efficient federated edge learning: Design and convergence analysis,'' \emph{IEEE Trans. Wireless Commun.}, vol.~20, no.~3, pp. 2120--2135, Mar. 2021.

\bibitem{10552192}
J.~Yao \emph{et~al.}, ``Wireless federated learning over resource-constrained networks: Digital versus analog transmissions,'' \emph{IEEE Trans. Wireless Commun.}, vol.~23, no.~10, pp. 14\,020--14\,036, Oct. 2024.

\bibitem{cover1991elements}
T.~M. Cover and J.~A. Thomas, \emph{Elements of Information Theory}.\hskip 1em plus 0.5em minus 0.4em\relax Wiley-Interscience, 1991.

\bibitem{5703199}
Y.~Xi, A.~Burr, J.~Wei, and D.~Grace, ``A general upper bound to evaluate packet error rate over quasi-static fading channels,'' \emph{IEEE Trans. Wireless Commun.}, vol.~10, no.~5, pp. 1373--1377, May 2011.

\bibitem{10551685}
R.~Wang, L.~Yang, T.~Tang, B.~Yang, and D.~Wu, ``Robust federated learning for heterogeneous clients and unreliable communications,'' \emph{IEEE Trans. Wireless Commun.}, vol.~23, no.~10, pp. 13\,440--13\,455, Oct. 2024.

\bibitem{10253642}
P.~Zheng, Y.~Zhu, Y.~Hu, Z.~Zhang, and A.~Schmeink, ``Federated learning in heterogeneous networks with unreliable communication,'' \emph{IEEE Trans. Wireless Commun.}, vol.~23, no.~4, pp. 3823--3838, Apr. 2024.

\bibitem{9756506}
M.~Shirvanimoghaddam, A.~Salari, Y.~Gao, and A.~Guha, ``Federated learning with erroneous communication links,'' \emph{IEEE Commun. Lett.}, vol.~26, no.~6, pp. 1293--1297, Jun. 2022.

\bibitem{pmlr-v235-qin24a}
Z.~Qin, D.~Chen, B.~Qian, B.~Ding, Y.~Li, and S.~Deng, ``Federated full-parameter tuning of billion-sized language models with communication cost under 18 kilobytes,'' in \emph{Proc. Int. Conf. Mach. Learn. (ICML)}, vol. 235, Vienna, Austria., Jul. 2024, pp. 41\,473--41\,497.

\bibitem{9277666}
S.~Zheng, C.~Shen, and X.~Chen, ``Design and analysis of uplink and downlink communications for federated learning,'' \emph{IEEE J. Sel. Areas Commun.}, vol.~39, no.~7, pp. 2150--2167, Jul. 2021.

\bibitem{9382094}
N.~Zhang and M.~Tao, ``Gradient statistics aware power control for over-the-air federated learning,'' \emph{IEEE Trans. Wireless Commun.}, vol.~20, no.~8, pp. 5115--5128, Aug. 2021.

\bibitem{Newton1967TheMP}
R.~L. Burden and J.~D. Faires, \emph{Numerical Analysis}.\hskip 1em plus 0.5em minus 0.4em\relax Boston, MA, USA: Cengage Learning, 2016.

\bibitem{cvx}
M.~Grant and S.~Boyd, ``{CVX}: Matlab software for disciplined convex programming, version 2.1,'' \url{https://cvxr.com/cvx}, Mar. 2014.

\bibitem{nocedal2006numerical}
J.~Nocedal and S.~J. Wright, \emph{Numerical Optimization}.\hskip 1em plus 0.5em minus 0.4em\relax New York, NY, USA: Springer, 2006.

\bibitem{LOBO1998193}
M.~S. Lobo, L.~Vandenberghe, S.~Boyd, and H.~Lebret, ``Applications of second-order cone programming,'' \emph{Linear Algebra Appl.}, vol. 284, no.~1, pp. 193--228, Nov. 1998.

\bibitem{roy2024elements}
D.~Roy and G.~V. Rao, \emph{Elements of classical and geometric optimization}.\hskip 1em plus 0.5em minus 0.4em\relax CRC Press, 2024.

\bibitem{8870236}
G.~Zhu, Y.~Wang, and K.~Huang, ``Broadband analog aggregation for low-latency federated edge learning,'' \emph{IEEE Trans. Wireless Commun.}, vol.~19, no.~1, pp. 491--506, 2020.

\bibitem{10146443}
A.~Bereyhi, A.~Vagollari, S.~Asaad, R.~R. Müller, W.~Gerstacker, and H.~V. Poor, ``Device scheduling in over-the-air federated learning via matching pursuit,'' \emph{IEEE Trans. Signal Process.}, vol.~71, pp. 2188--2203, Jun. 2023.

\bibitem{8664630}
S.~Wang \emph{et~al.}, ``Adaptive federated learning in resource constrained edge computing systems,'' \emph{IEEE J. Sel. Areas Commun.}, vol.~37, no.~6, pp. 1205--1221, Jun. 2019.

\bibitem{9337227}
M.~M. Amiri, D.~Gündüz, S.~R. Kulkarni, and H.~V. Poor, ``Convergence of update aware device scheduling for federated learning at the wireless edge,'' \emph{IEEE Trans. Wireless Commun.}, vol.~20, no.~6, pp. 3643--3658, Jun. 2021.

\bibitem{10318063}
M.~Ma, V.~W. Wong, and R.~Schober, ``Channel-aware joint {AoI} and diversity optimization for client scheduling in federated learning with {Non-IID} datasets,'' \emph{IEEE Trans. Wireless Commun.}, vol.~23, no.~6, pp. 6295--6311, Jun. 2024.

\end{thebibliography}

\end{document}